\documentclass[letterpaper, 10 pt, conference]{bib/ieeeconf}
\IEEEoverridecommandlockouts                              
\usepackage[normalem]{ulem}	                        
\usepackage[table,usenames,dvipsnames]{xcolor}      
\usepackage{extarrows}                              
\let\labelindent\relax
\usepackage{enumitem}                               

\usepackage{amsmath,amsthm,amssymb,amsfonts,dsfont} 
\usepackage{algorithm,algorithmicx,listings}        
\usepackage[noend]{algpseudocode}			        
\usepackage{graphicx}
\usepackage{subfigure}
\usepackage{tabularx}
\usepackage{amsmath}
\usepackage{multirow,multicol, array}
\usepackage[font={small}]{caption}   
\let\appendices\relax
\usepackage[titletoc,title]{appendix}
\usepackage[breaklinks=true, colorlinks, bookmarks=true, citecolor=Black, urlcolor=Violet,linkcolor=Black]{hyperref}

\def\argmin{\mathop{\arg\min}\limits}	%
\newcommand{\indicator}{\mathds{1}}

\newcommand{\prl}[1]{\mathopen{}\left(#1\right)\mathclose{}}
\newcommand{\brl}[1]{\mathopen{}\left[#1\right]\mathclose{}}
\newcommand{\crl}[1]{\mathopen{}\left\{#1\right\}\mathclose{}}
\newcommand{\scaleMathLine}[2][1]{\resizebox{#1\linewidth}{!}{$\displaystyle{#2}$}}
\newcommand{\T}{\mathsf{T}}

\newtheorem{proposition}{Proposition}

\newtheorem*{assumption*}{Assumption}

\newtheorem*{remark*}{Remark}
\newtheorem*{problem*}{Problem}
\newtheorem{problem}{Problem}

\title{\LARGE \bf
Search-based Motion Planning for Quadrotors using \\
Linear Quadratic Minimum Time Control
}
\author{Sikang Liu, Nikolay Atanasov, Kartik Mohta, and Vijay Kumar
\thanks{This work is supported in part by ARL \# W911NF-08-2-0004, DARPA \# HR001151626/HR0011516850, ARO \# W911NF-13-1-0350, and ONR \# N00014-07-1-0829. The authors are with the GRASP Laboratory, University of Pennsylvania. Email: \texttt{\{sikang, atanasov, kmohta, kumar\}@seas.upenn.edu}}
}

\begin{document}
\maketitle
\begin{abstract}
In this work, we propose a search-based planning method to compute dynamically feasible trajectories for a quadrotor flying in an obstacle-cluttered environment. Our approach searches for smooth, minimum-time trajectories by exploring the map using a set of short-duration motion primitives. The primitives are generated by solving an optimal control problem and induce a finite lattice discretization on the state space which can be explored using a graph-search algorithm. The proposed approach is able to generate resolution-complete (i.e., optimal in the discretized space), safe, dynamically feasibility trajectories efficiently by exploiting the explicit solution of a Linear Quadratic Minimum Time problem. It does not assume a hovering initial condition and, hence, is suitable for fast online re-planning while the robot is moving. Quadrotor navigation with online re-planning is demonstrated using the proposed approach in simulation and physical experiments and comparisons with trajectory generation based on state-of-art quadratic programming are presented.
\end{abstract}

\section{Introduction}
\label{sec:intro}

Smooth trajectories obtained by minimizing jerk or snap have been widely used to control differentially flat dynamical systems such as quadrotors~\cite{mellingerICRA2011,hehn2011quadrocopter,mueller2015}. These trajectories are represented via time-parameterized polynomials, which converts the trajectory generation problem into one of finding polynomial coefficients that satisfy certain constraints. Recent work exploring time-optimal trajectory generation includes~\cite{bouktir2008,Jamieson2016}. If additionally, obstacle avoidance is added as a consideration, the trajectory generation problem becomes more challenging. While mixed integer optimization techniques~\cite{mellingerICRA2012,deitsICRA2015} handle collisions reliably, they suffer from high computational costs. Recent work demonstrated practical application of quadratic programming~\cite{richter2016polynomial,liu2016high,liu2017plan,chen2016online} to derive collision-free trajectories in real-time. These methods separate the trajectory generation problem in two parts: (i) planning a collision-free geometric path and (ii) optimizing it locally to obtain a dynamically-feasible time-parametrized trajectory. In this way, one can solve for a locally optimal trajectory with respect to a given \textit{time allocation}. However, the prior geometric path restricts the generated trajectory to be inside a given homology class which may not contain a globally optimal (or even feasible) trajectory (Fig.~\ref{fig: sim}).

\begin{figure}[t]
\centering
    \includegraphics[width=0.8\linewidth]{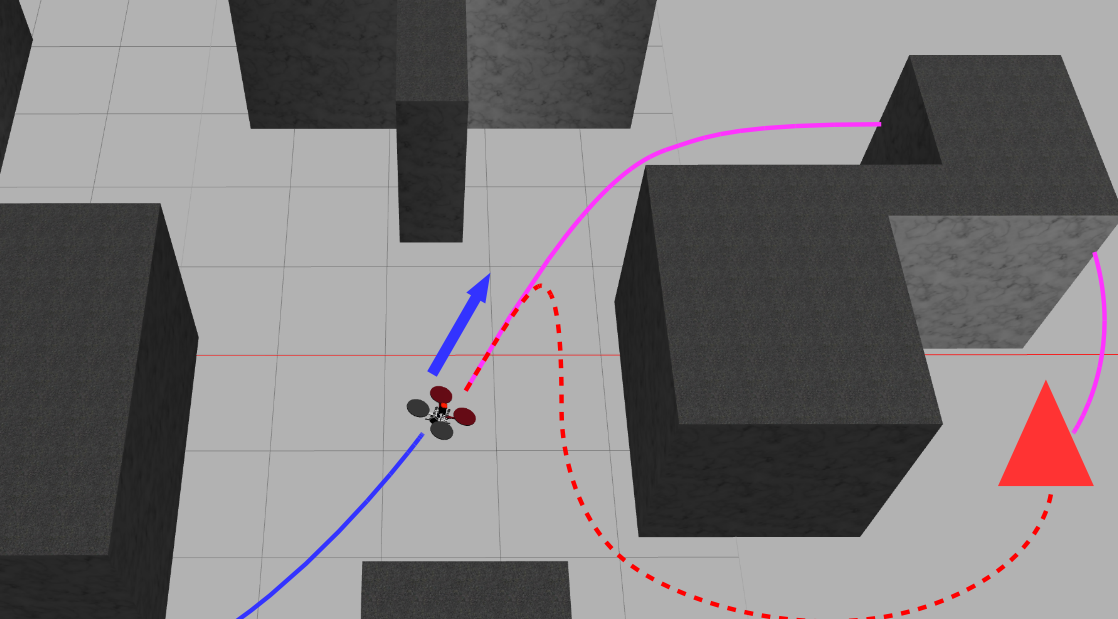}
     \caption{Taking the quadrotor dynamics into account is important for obtaining a smooth trajectory (magenta) while flying at non-zero velocity towards a goal (red triangle). In contrast, existing methods generate a trajectory (red dashed curve) from a shortest path that ignores the system dynamics. Instead of relying on a prior shortest path, the approach proposed in this paper plans globally-optimal trajectories based on time and control efforts. \label{fig: sim}}
\end{figure}

This paper proposes an approach for global trajectory optimization that obtains collision-free, dynamically-feasible, minimum-time, smooth trajectories in real time. Instead of using a geometric path as a prior, our approach explores the space of trajectories using a set of short-duration motion primitives generated by solving an optimal control problem. We prove that the primitives induce a finite lattice discretization on the state space, which can in turn be explored using a graph-search algorithm. It is well-known that the graph search in high-dimensional state spaces is not computationally efficient because there are many states to be explored. However, with the help of a tight lower bound (heuristic) on the optimal cost we can inform and significantly accelerate the search. The main contribution of this paper can be concluded as: 
\begin{enumerate}
\item generation of motion primitives that convert an optimal control problem to graph search
\item a search heuristic(s) based on the explicit solution of a Linear Quadratic Minimum Time problem
\end{enumerate}


In contrast with previous works based on motion primitives like~\cite{likhachev2009planning},~\cite{macallister2013path},~\cite{pivtoraikoICRA2013}, our approach does not require a big precomputed look-up table to find connections between different graph nodes. To reduce the run time, we propose to plan a trajectory in a lower dimension state space and refine a final trajectory that is executable by quadrotors through an unconstrained quadratic programming. We also show that our method generates smoother trajectories compared to the traditional path-based trajectory generation approaches. We demonstrate that our approach can be used for online re-planning during fast quadrotor navigation in various cluttered environments. The the code used in this work is open sourced on \url{https://github.com/sikang/motion_primitive_library}.

\section{Problem Formulation}
\label{sec:problem}
Let $x(t) \in \mathcal{X} \subset \mathbb{R}^{3n}$ be a dynamical system state, consisting of 3-D position and its $(n-1)$ derivatives (velocity, acceleration, jerk, etc.). Let $\mathcal{X}^{free}\subset\mathcal{X}$ denote free region of the state space that, in addition to capturing the obstacle-free positions $\mathcal{P}^{free}$, also specifies constraints $\mathcal{D}^{free}$ on the system's dynamics, i.e., maximum velocity $v_{max}$, acceleration $a_{max}$, and higher order derivatives in each axis. Note that $\mathcal{P}^{free}$ is bounded by the size of the map that we are planning in. Thus, $\mathcal{X}^{free} := \mathcal{P}^{free} \times \mathcal{D}^{free} = \mathcal{P}^{free} \times [-v_{max},v_{max}]^3 \times [-a_{max},a_{max}]^3 \times \ldots$. Denote the obstacle region as $\mathcal{X}^{obs}:=\mathcal{X} \setminus \mathcal{X}^{free}$.

As described in \cite{mellinger2012trajectory} and many other related works, the differential flatness of quadrotor systems allow us to construct control inputs from 1-D time-parametrized polynomial trajectories specified independently in each of the three position axes. Thus, we consider polynomial state trajectories $x(t):= [p_D(t)^\T,\dot{p}_D(t)^\T,\ldots,p_D^{(n-1)}(t)^\T]^\T$, where
\begin{equation}
\label{eq:poly}
p_D(t) := \sum_{k=0}^{K} d_k\frac{t^k}{k!} = d_K\frac{t^K}{K!} + \ldots + d_1 t + d_0 \in \mathbb{R}^3
\end{equation}
and $D := \brl{d_0, \ldots, d_K} \in \mathbb{R}^{3 \times (K+1)}$. To simplify the notation, we denote the system's velocity by $v(t) := \dot{p}_D^\T (t)$, acceleration by $a(t) := \ddot{p}_D^\T (t)$, jerk by $j(t) := \dddot{p}_D^\T(t)$, etc., and drop the subscript $D$ where convenient. Polynomial trajectories of the form Eq.~\eqref{eq:poly} can be generated by considering a linear time-invariant dynamical system $p_D^{(n)}(t) = u(t)$, where the control input is $u(t) \in \mathcal{U} := [-u_{max},u_{max}]^3 \subset \mathbb{R}^3$. In state space form, we obtain a system as
\begin{align}
\label{eq:sys}
\dot{x} &= Ax+Bu\nonumber\\
A &= \begin{bmatrix}
    \mathbf{0} & \mathbf{I}_3 & \mathbf{0} & \cdots &\mathbf{0}\\
    \mathbf{0} & \mathbf{0} & \mathbf{I}_3 & \cdots &\mathbf{0}\\
    \vdots& \ddots&\ddots &\ddots&\vdots\\
    \mathbf{0}& \cdots & \cdots & \mathbf{0} & \mathbf{I}_3\\
    \mathbf{0}& \cdots & \cdots & \mathbf{0} & \mathbf{0}\\
    \end{bmatrix}, \quad B = \begin{bmatrix} \mathbf{0} \\ \mathbf{0}\\\vdots\\\mathbf{0}\\\mathbf{I}_3\end{bmatrix}
\end{align}
We are interested in planning state trajectories that are \textit{collision-free}, respect the constraints on the dynamics, and are \textit{minimum-time} and \textit{smooth}. We define the \textit{smoothness} or \textit{effort} of a trajectory as the square $L^2$-norm of the control input $u(t)$:
\begin{equation}
\label{eq: J}
J(D) := \int_0^T \left\|u(t)\right\|^2 dt = \int_0^T \left\|p_D^{(n)}(t)\right\|^2 dt
\end{equation}
and consider the following problem.

\begin{problem}
\label{prob:1}
Given an initial state $x_0 \in \mathcal{X}^{free}$ and a goal region $\mathcal{X}^{goal} \subset \mathcal{X}^{free}$, find a polynomial trajectory parametrization $D\in \mathbb{R}^{3 \times (K+1)}$ and a time $T\geq 0$ such that:
\begin{equation}
  \label{eq:problem1}
  \begin{gathered}
    \min_{D,T} \; J(D) + \rho T\\
    \begin{aligned}
      \text{s.t.}\;&\dot{x}(t) = Ax(t)+Bu(t), \quad \forall \, t \in [0, T]\\
      &x(0) = x_0, \quad x(T) \in \mathcal{X}^{goal}\\
      &x(t) \in \mathcal{X}^{free},\quad u(t)\in{\mathcal{U}}, \quad \forall \, t \in [0, T]
    \end{aligned}
  \end{gathered}
\end{equation}
where the parameter $\rho \geq 0$ determines the relative importance of the trajectory duration $T$ versus its smoothness $J$.
\end{problem}

In the remainder, we denote the optimal cost from an initial state $x_0$ to a goal region $\mathcal{X}^{goal}$ by $C^*\prl{x_0,\mathcal{X}^{goal}}$. The reason for choosing such an objective function is illustrated in Fig.~\ref{fig: trajs}. This problem is a Linear Quadratic Minimum-Time problem~\cite{LQMT} with state constraints, $x(t) \in \mathcal{X}^{free}$, and input constraints, $u(t)\in{\mathcal{U}}$. As the derivation in Sec.~\ref{sec:heur} shows, if we drop the constraints $x(t) \in \mathcal{X}^{free},u(t)\in{\mathcal{U}}$, the optimal solution can be obtained via Pontryagin's minimum principle~\cite{LQMT,lewisBook} and the optimal choice of polynomial degree is $K=2n-1$. The main challenge is the introduction of the constraints $x(t) \in \mathcal{X}^{free},u(t)\in{\mathcal{U}}$. In this paper, we show that these safety constraints can be handled by converting the problem to a deterministic shortest path problem~\cite[Ch.2]{OptimalControlBook} with a $3n$ dimensional state space $\mathcal{X}$ and a $3$ dimensional control space $\mathcal{U}$. Since the control space $\mathcal{U}$ is always $3$ dimensional, a search-based planning algorithm such as $A^*$~\cite{ara_star} that discretizes $\mathcal{U}$ using \textit{motion primitives} is efficient and resolution-complete (i.e., it can compute the optimal trajectory in the discretized space in finite-time, unlike sampling-based planners such as RRT~\cite{Lavalle98rapidly-exploringrandom,rrt-star}).

\begin{figure}[t]
  \subfigure[$T = 4, J = 19$.]{\includegraphics[width=0.32\columnwidth, trim=0 0 0 0, clip=true]{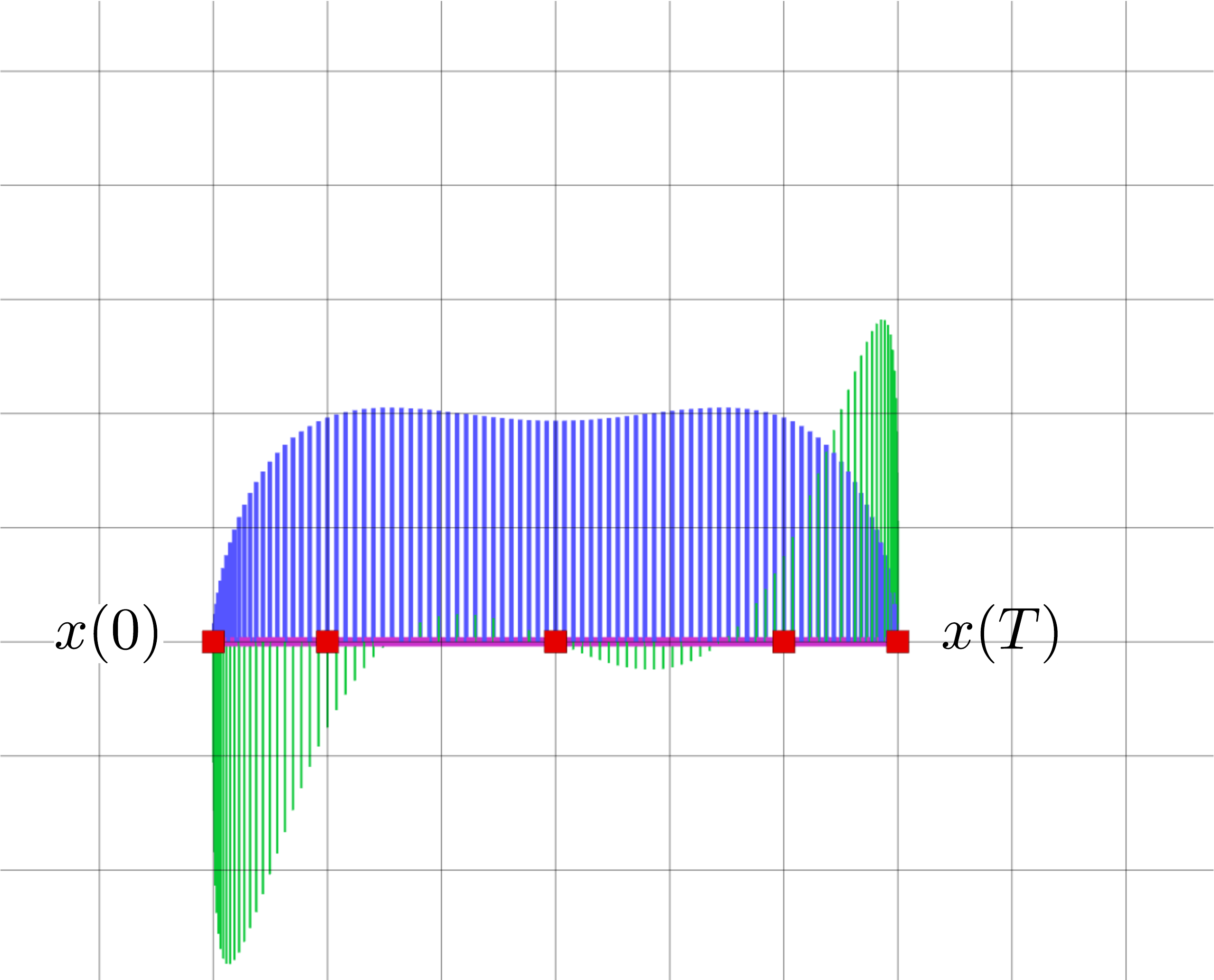}\label{fig: traj1}}
  \subfigure[$T = 4, J = 48$.]{\includegraphics[width=0.32\columnwidth, trim=0 0 0 0, clip=true]{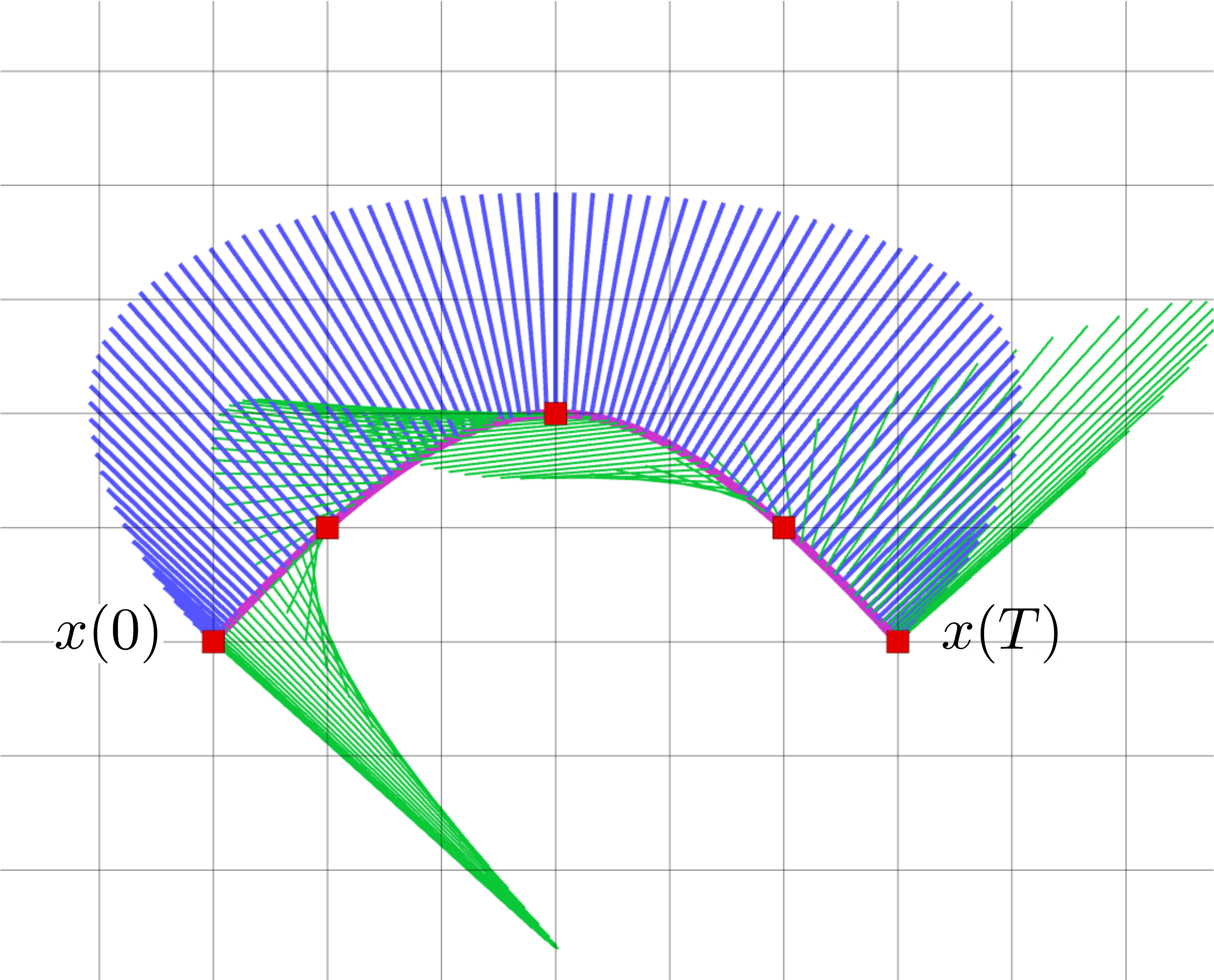}\label{fig: traj2}}
  \subfigure[$T = 7, J = 5$.]{\includegraphics[width=0.32\columnwidth, trim=0 0 0 0, clip=true]{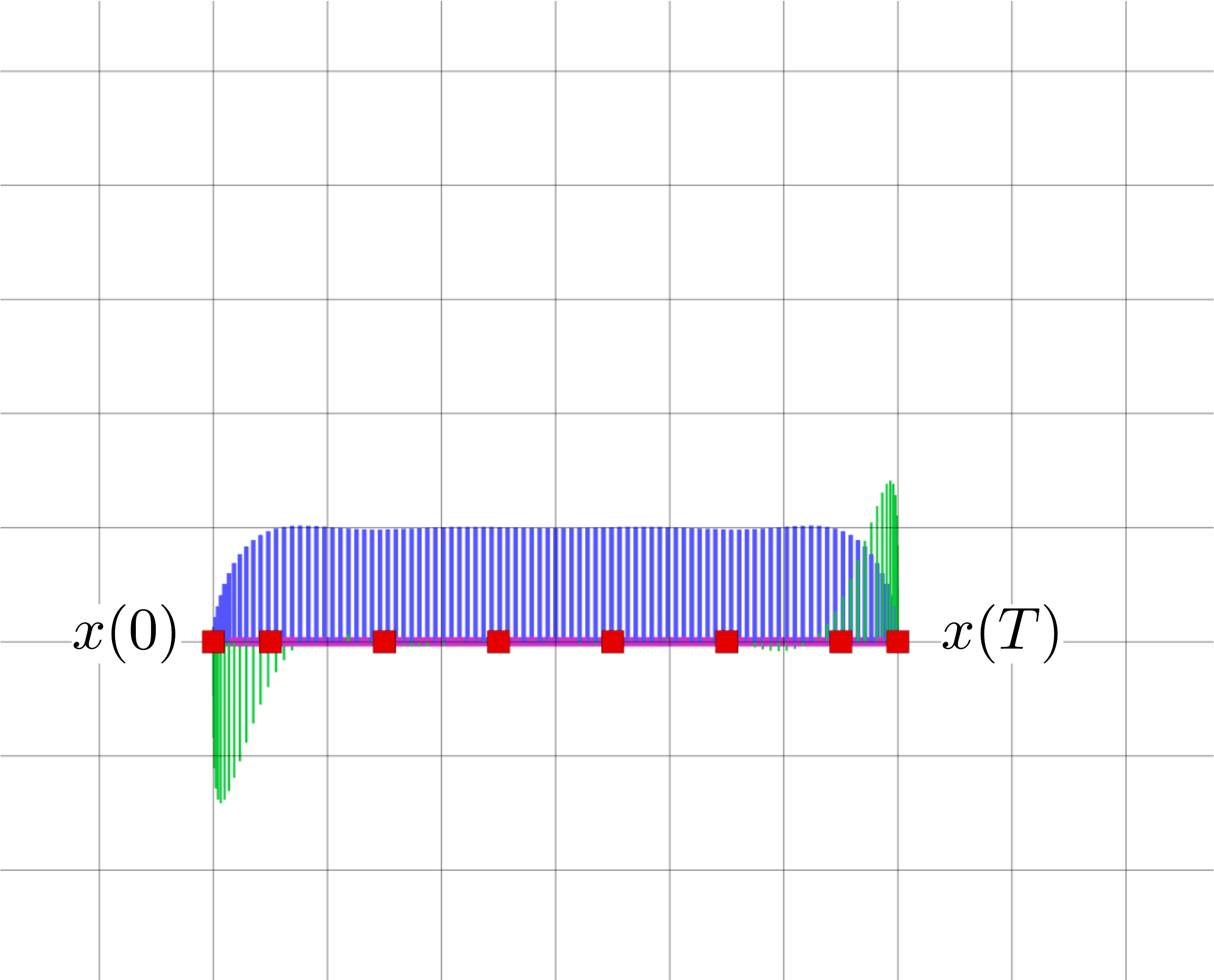}\label{fig: traj3}}
   \caption{Three trajectories start from $x(0)$ to $x(T)$. Blue and green rays indicate the magnitude of velocity and acceleration along trajectories respectively. If the effort $J$ is disregarded, i.e.\ $\rho \rightarrow \infty$ in Eq.~\eqref{eq:problem1}, trajectories (a) and (b) have equivalent cost of $T = 4$. If the time $T$ is not considered, i.e.\ $\rho=0$, trajectory (c) become optimal. Since we are interested in low-effort trajectories, $\rho$ should not be infinite (so that (a) is preferable to (b)) but it should still be large enough to prioritize fast trajectories. Thus, in this comparison, (a) is preferable to both (b) and (c). \label{fig: trajs}}
\end{figure}
\section{Optimal Trajectory Planning}
\label{sec:planning}

\subsection{Motion Primitives}
First, we discuss the construction of motion primitives for the system in Eq.~\eqref{eq:sys} that will allow us to convert Problem~\ref{prob:1} from an optimal control problem to a graph-search problem. Instead of using the control set $\mathcal{U}$, we consider a lattice discretization~\cite{pivtoraiko2009differentially} $\mathcal{U}_M:=\{u_1,\ldots,u_M\} \subset \mathcal{U}$, where each control $u_m \in \mathbb{R}^3$ vector will define a motion of short duration for the system. One way to obtain the discretization $\mathcal{U}_M$ is to choose a number of samples $\mu\in \mathbb{Z}^+$ along each axis $[0,u_{max}]$, which defines a discretization step $d_u := \frac{u_{max}}{\mu}$ and results in $M = (2\mu+1)^3$ motion primitives. Given an initial state $x_0 := [p_0^\T,\;v_0^\T,\;a_0^\T,\ldots]^\T$, we generate a motion primitive of duration $\tau > 0$ that applies a constant control input $u(t) \equiv u_m \in \mathcal{U}_M$ for $t \in [0,\tau]$ so that:
\[
u(t) = p_D^{(n)}(t) = \sum_{k=0}^{K-n} d_{k+n} \frac{t^k}{k!} \equiv u_m.
\]
The control input being constant, implies that all coefficients that involve time need to be identically zero, i.e.:
\[
  d_{(n+1):K} = \mathbf{0} \implies u_m = d_n
\]
Integrating the control expression $u(t) = u_m$ with an initial condition $x_0$ results in
\[
p_D(t) = u_m \frac{t^n}{n!} + \ldots + a_0 \frac{t^2}{2} + v_0t + p_0
\]
or, equivalently, the resulting trajectory of the linear time-invariant system in Eq.~\eqref{eq:sys} is:
\[
  x(t) = \underbrace{e^{At}}_{F(t)} x_0 + \underbrace{\brl{\int_{0}^t e^{A(t-\sigma)}Bd\sigma}}_{G(t)} u_m
\]
An example of the resulting system trajectories is given in Fig.~\ref{fig:prs}.
\begin{figure}[htp]
  \centering
  \subfigure[Discretized Acceleration.]{\includegraphics[width=0.4\columnwidth, trim=0 0 0 0, clip=true]{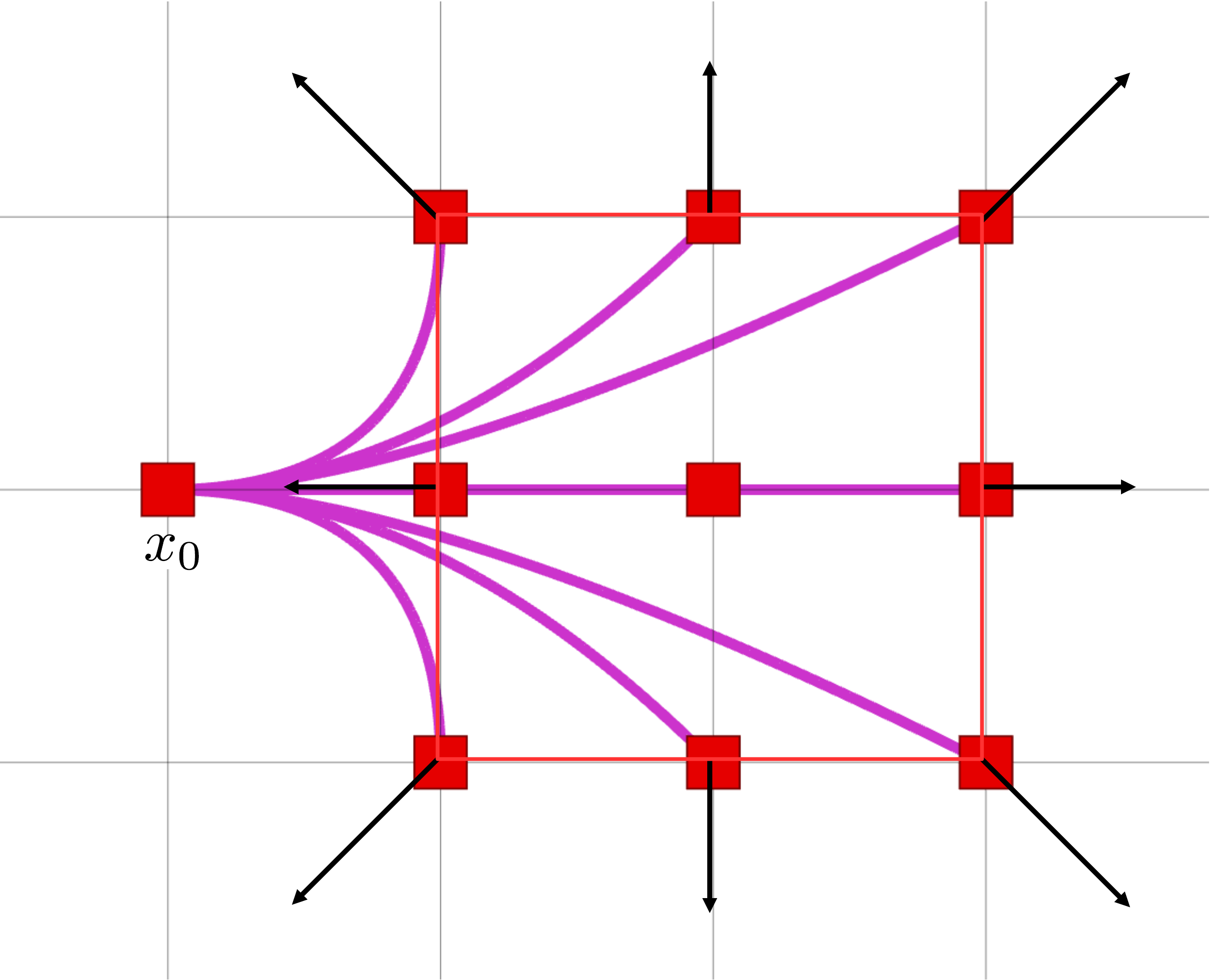}\label{fig: vel}}
  \subfigure[Discretized Jerk.]{\includegraphics[width=0.4\columnwidth, trim=0 0 0 0, clip=true]{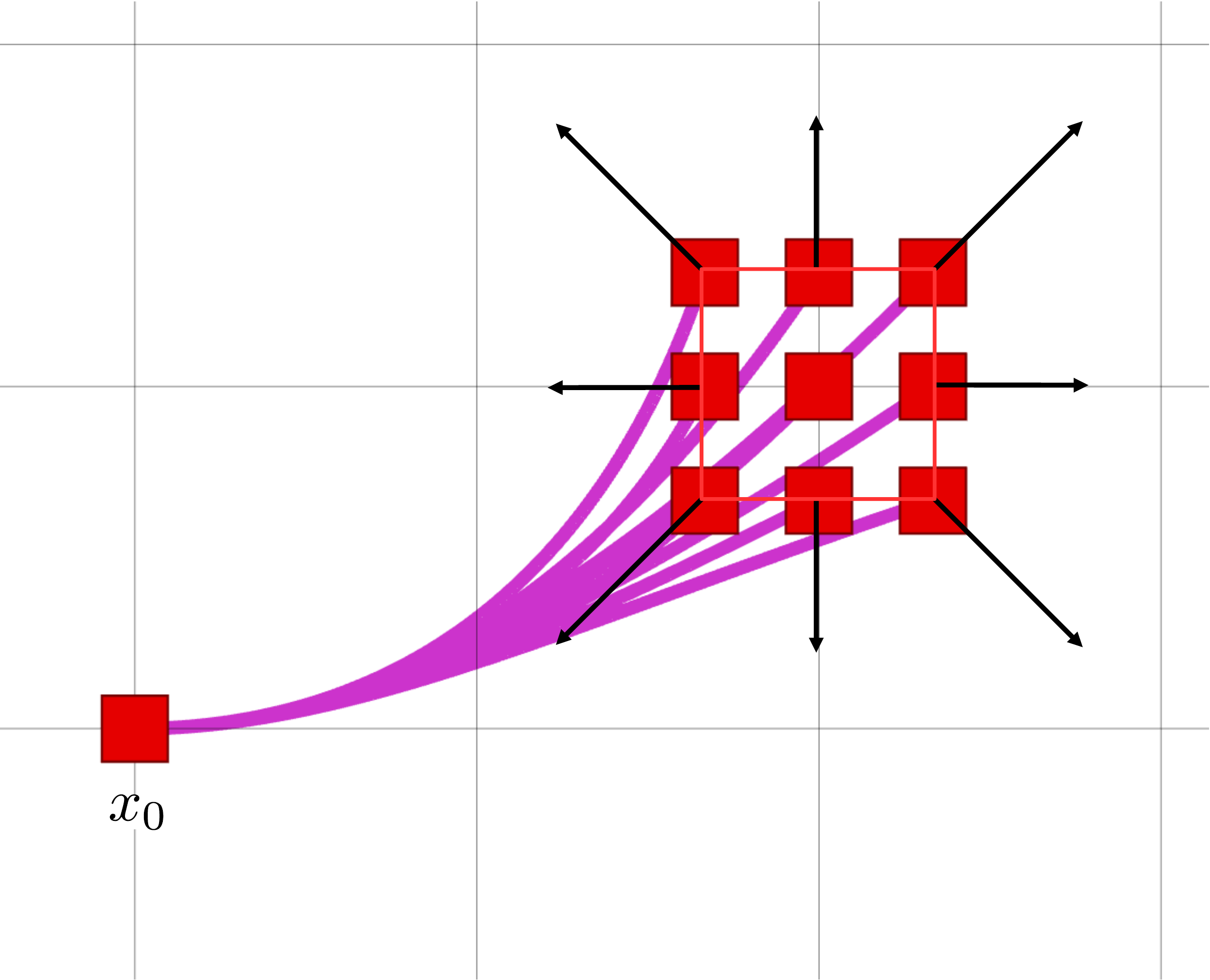}\label{fig: acc}}
   \caption{Example of 9 planar motion primitives from initial state $x_0$ for an acceleration-controlled ($n=2$) system (left) and a jerk-controlled ($n=3$) system (right). The black arrow indicates correpsonding control input. The red boundary shows the feasible region for the end states (red squares), which is induced by the control limit $u_{max}$. The initial velocity and acceleration are $v_0 = [1, 0, 0]^\T$ and $a_0 = [0, 1, 0]^\T$ (only for the right figure).\label{fig:prs}}
\end{figure}
Since both the duration $\tau$ and the control input $u_m$ are fixed, the cost of the motion primitive according to Eq.~\eqref{eq:problem1} is $\prl{\|u_m\|^2+\rho}\tau \label{eq:C}$

\subsection{Induced Space Discretization}
\label{sec:dis}

\begin{proposition}
  \label{prop:discretization}
  The motion primitives defined in the previous section induce a discretization on the state space $\mathcal{X}$.
\end{proposition}
\begin{proof}
See App.~\ref{app:a}.
\end{proof}

This discretization of the state space allows us to construct a graph representation of the reachable system states by starting at $x_0$ and applying all primitives to obtain the $M$ possible states after a duration of $\tau$ (see Fig.~\ref{fig:prs} and Alg.~\ref{alg:succ}). Applying all possible primitives to each of the $M$ states again, will result in $M^2$ possible states at time $2\tau$. Since the free space $\mathcal{X}^{free}$ is bounded and discretized, the set of reachable states $\mathcal{S}$ is finite.

This defines a graph $\mathcal{G}(\mathcal{S}, \mathcal{E})$, where $\mathcal{S}$ is the discrete set of reachable system states and $\mathcal{E}$ is the set of edges that connect states in the graph, each defined by a motion primitive $e := (u_m,\tau)$. Let $s_0$ be the state corresponding to $x_0$.

We use Algorithm~\ref{alg:succ} to explore the free state space $\mathcal{X}^{free}$ and build the connected graph: in line 4, the primitive is calculated using the fully defined state $s$ and a control input $u_m$ given the constant time $\tau$; line 5 checks the feasibility of the primitive, this step will be further discussed in Section.~\ref{sec:col}; in line 6, we evaluate the end state of a valid primitive and add it to the set of successors of the current node; in the meanwhile, we estimate the edge cost from the corresponding primitive. After checking through all the primitives in the finite control input set, we add the nodes in successor set $\mathcal{R}(s)$ to the graph, and we continue expanding until we reach the goal region.
\begin{algorithm}[t]
\caption{Given $s\in \mathcal{S}$ and a motion primitive set $\mathcal{U}_M$ with duration $\tau$, find the states $\mathcal{R}(s)$ that are reachable from $s$ in one step and their associated costs $\mathcal{C}(s)$. \label{alg:succ}}
\begin{algorithmic}[1]
\Function{GetSuccessors}{$s,\mathcal{U}_M, \tau$}
\State $\mathcal{R}(s) \gets \emptyset, \quad \mathcal{C}(s) \gets \emptyset$
\ForAll{$u_m \in \mathcal{U}_M$}
\State $e_m(t) \gets F(t)s + G(t)u_m, \quad t\in[0,\tau]$
\If {$e_m(t) \subset \mathcal{X}^{free}$}
\State $s_m \gets e_m(\tau)$
\State $\mathcal{R}(s) \gets \mathcal{R}(s) \cup \{s_m\}$
\State $\mathcal{C}(s) \gets \mathcal{C}(s) \cup \{ \prl{\|u_m\|^2+\rho}\tau\}$
\EndIf
\EndFor
\State\Return {$\mathcal{R}(s),\mathcal{C}(s)$}
\EndFunction
\end{algorithmic}
\end{algorithm}

\begin{proposition}\label{prop:2}
The motion primitive $u_{ij} \in \mathcal{U}_M$ which connects two consecutive states $s_i,s_j \in \mathcal{S}$ with $s_j = F(\tau)s_i + G(\tau) u_{ij}$ is optimal according to the cost function in Eq.~\eqref{eq:problem1}.
\end{proposition}

\begin{proof}
See App.~\ref{app:b}.
\end{proof}

\subsection{Deterministic Shortest Trajectory}
Given the set of motion primitives $\mathcal{U}_M$ and the induced space discretization discussed in the previous section, we can re-formulate Problem~\ref{prob:1} as a graph-search problem. This can be done by introducing additional constraints that stipulate that the control input $u(t)$ in Eq.\eqref{eq:problem1} is piecewise-constant over intervals of duration $\tau$. More precisely, we introduce an additional variable $N\in \mathbb{Z}^+$, such that $T = N\tau$, and $u_k \in \mathcal{U}_M$ for $k = 0, \ldots, N-1$ and a constraint in Eq.~\eqref{eq:problem1}:
\[
u(t) = \sum_{k=0}^{N-1} u_k \indicator_{\{t \in [k\tau,(k+1)\tau)\}}
\]
that forces the control trajectory to be a composition of the motion primitives in $\mathcal{U}_M$. This leads to the following deterministic shortest path problem~\cite[Ch.2]{OptimalControlBook}.

\begin{problem}\label{prob:2}
  Given an initial state $x_0 \in \mathcal{X}^{free}$, a goal region $\mathcal{X}^{goal} \subset \mathcal{X}^{free}$, and a finite set of motion primitives $\mathcal{U}_M$ with duration $\tau > 0$, choose a sequence of motion primitives $u_{0:N-1}$ of length $N$ such that:
  \begin{equation}
    \label{eq:problem2}
    \begin{gathered}
      \min_{N, u_{0:N-1}} \; \prl{\sum_{k=0}^{N-1}  \|u_k\|^2 + \rho N}\tau\\
      \begin{aligned}
        \text{s.t.}\;&x_k(\tilde{t}) = F(\tilde{t})s_k+G(\tilde{t})u_k \subset \mathcal{X}^{free}, \quad \tilde{t}\in[0,\tau]\\
        &x_{k}(\tilde{t})\subset \mathcal{X}^{free} \quad \forall \, k \in \crl{0,\ldots,N-1}, \; \tilde{t}\in[0,\tau]\\
        &s_{k+1} = x_k(\tau), \quad \forall \, k \in \crl{0,\ldots,N-1}\\
        &s_0 = x_0, \quad s_{N} \in \mathcal{X}^{goal}\\
        &u_{k} \in \mathcal{U}_M, \quad \forall \, k \in \crl{0,\ldots,N-1}
      \end{aligned}
    \end{gathered}
  \end{equation}
\end{problem}

The optimal cost of Problem~\ref{prob:2} is an upper bound to the optimal cost of Problem~\ref{prob:1} because Problem~\ref{prob:2} is just a constrained version of Problem~\ref{prob:1}. However, this re-formulation to discrete control and state-spaces enables an efficient solution. Such problems can be solved via search-based~\cite{hart1968formal,ara_star} or sampling-based~\cite{Lavalle98rapidly-exploringrandom,rrt-star,arslan2013use} motion planning algorithms. Since only the former guarantees finite-time (sub-)optimality, we use an $A^*$ method and focus on the design of an accurate, consistent heuristic and efficient, guaranteed collision checking methods in following subsections.

\subsection{Heuristic Function Design}
\label{sec:heur}
Devising an efficient graph search for solving Problem~\ref{prob:2} requires an approximation of the optimal cost function, i.e., a heuristic function, that is admissible\footnote{\label{ftn:heur}A heuristic function $h$ is \textit{admissible} if it underestimates the optimal cost-to-go from $x_0$, i.e., $0 \leq h(x_0) \leq C^*\prl{x_0,\mathcal{X}^{goal}}, \forall x_0 \in \mathcal{X}$.}, informative (i.e., provides a tight approximation of the optimal cost), and consistent\footnote{\label{ftn:heur2}A heuristic function $h$ is \textit{consistent} if it satisfies the triangle inequality, i.e., $h(x_0) \leq C^*\prl{x_0,\{x_1\}} + h(x_1), \forall x_0,x_1 \in \mathcal{X}$.} (i.e., can be inflated in order to obtain solutions with bounded suboptimality very efficiently~\cite{ara_star}). Since by construction, the optimal cost of Problem~\ref{prob:2} is bounded below by the optimal cost of Problem~\ref{prob:1}, we can obtain a good heuristic function by solving a relaxed version of Problem~\ref{prob:1}. Our idea is to replace constraints in Eq.~\eqref{eq:problem1} that are difficult to satisfy, namely, $x(t) \in \mathcal{X}^{free}$ and $u(t) \in \mathcal{U}$, with a constraint on the time $T$. In this section, we show that such a relaxation of Problem~\ref{prob:1} can be solved optimally and efficiently.

\subsubsection{Minimum Time Heuristic}
Intuitively, the constraints on maximum velocity, acceleration, jerk, etc. due to $\mathcal{X}^{obs}$ and $\mathcal{U}$ induce a lower bound $\bar{T}$ on the minimum achievable time in \eqref{eq:problem1}. For example, since the system's maximum velocity is bounded by $v_{max}$ along each axis, the minimum time for reaching the closest state $x_f$ in the goal region $\mathcal{X}^{goal}$ is bounded below by $\bar{T}_v:=\frac{\|p_f-p_0\|_\infty}{v_{max}}$. Similarly, since the system's maximum acceleration is bounded by $a_{max}$, the state $x_f:=[p_f^\T,v_f^\T]^\T$ cannot be reached faster than:
\begin{equation*}
  \begin{gathered}
    \min_{\bar{T}_a,a(t)} \; \bar{T}_a \\
    \begin{aligned}
      \text{s.t.}\quad & \|a(t)\| \leq a_{max}, \quad \forall \, t \in [0, T]\\
      & p(0) = p_0, \; v(0) = v_0\\
      &p(\bar{T}_a) = p_f, \; v(\bar{T}_a) = v_f
    \end{aligned}
  \end{gathered}
\end{equation*}
The above is a minimum-time (Brachistochrone) optimal control problem with input constraints, which may be difficult to solve directly in 3-D~\cite{mintime_acc_constraint} but can be solved in closed-form along individual axes~\cite[Ch.5]{lewisBook} to obtain lower bounds $\bar{T}_a^x$, $\bar{T}_a^y$, $\bar{T}_a^z$. This procedure can be continued for the constraint on jerk $j_{max}$ and  those on higher-order derivatives but the problems become more complicated to solve and the computed times are less likely to provide better bounds the previous ones. Hence, we can define a lower bound on the minimum achievable time via $\bar{T}:=\max\{\bar{T}_v,\bar{T}_a^x,\bar{T}_a^y,\bar{T}_a^z,\bar{T}_j,\ldots\}$ but for simplicity we use the easily computable but less tight bound $\bar{T} = \bar{T}_v$.

Hence, to find a heuristic function, we relax Problem~\ref{prob:1} by replacing the state and input constraints, $x(t) \in \mathcal{X}^{free}$ and $u(t) \in \mathcal{U}$, with the lower bound $T \geq \bar{T}_v$:
\begin{equation}
  \label{eq:LQMT}
  \begin{gathered}
    \min_{D,T} \; J(D) + \rho T\\
    \begin{aligned}
      \text{s.t.}\;&\dot{x}(t) = Ax(t)+Bu(t), \quad \forall  t \in [0, T]\\
      &x(0) = x_0, \quad x(T) \in \mathcal{X}^{goal}\\
      & T \geq \bar{T}
    \end{aligned}
  \end{gathered}
\end{equation}
Since $J(D) \geq 0$, a straight-forward way to obtain a lower-bound on the optimal cost is:
\[
C^*\prl{x_0,\mathcal{X}^{goal}} = J(D^*) + \rho T^* \geq \rho \bar{T}_v
\]
Hence, given nodes $s_0,s_f \in \mathcal{S}$ in the discretized space, the following is an admissible heuristic function:
\begin{equation}\label{eq:h1}
h_1(s_0) = \rho \bar{T}_v = \frac{\rho\|p_f-p_0\|_\infty}{v_{max}}
\end{equation}
for Problem~\ref{prob:2}. It is easy to see that it is also \textit{consistent} due to the triangle inequality for distances.

\subsubsection{Linear Quadratic Minimum Time}
While the minimum-time heuristic is very easy to compute and takes velocity constraints into account, it is not a very tight lower bound on the optimal cost in Eq.~\eqref{eq:problem2} because it disregards the control effort. The reason is that instead of solving Eq.~\eqref{eq:LQMT}, we simply found a lower bound in the previous subsection. An important observation is that after removing the constraints $x(t) \in \mathcal{X}^{free}$ and $u(t) \in \mathcal{U}$, the relaxed problem Eq.~\eqref{eq:LQMT} is in fact the classical Linear Quadratic Minimum-Time Problem~\cite{LQMT}. The optimal solution to Eq.~\eqref{eq:LQMT} can be obtained from \cite[Thm.2.1]{LQMT} with a minor modification introducing the additional constraint on time $T \geq \bar{T}$.
\begin{proposition}
\label{prop:LQMT}
Let $x_f\in  \mathcal{X}^{goal}$ be a fixed final state and define $\delta_T:= x_f - e^{AT}x_0$ and the controllability Gramian $W_T:= \int_0^T e^{At} BB^\T e^{A^\T t} dt$. Then, the optimal time $T$ in Eq.~\eqref{eq:LQMT} is either the lower bound $\bar{T}$ or the solution of following equation:
\begin{equation}
\scaleMathLine[0.9]{-\frac{d}{dT} \crl{\delta_T^\T W_T^{-1} \delta_T} = 2x_f^\T A^{\T}W_T^{-1}\delta_T + \delta_T^\T W_T^{-1}BB^\T W_T^{-1}\delta_T = \rho}
\end{equation}
The optimal control is:
\begin{equation}
u^*(t) := B^\T e^{A^\T(T-t)}W_{T}^{-1}\delta_{T}
\end{equation}
While the optimal cost is:
\begin{equation}\label{eq:optimal_cost}
h_2(x_0)=\delta_T^\T W_T^{-1} \delta_T+\rho T
\end{equation}
  The polynomial coefficients $D \in \mathbb{R}^{3 \times (2n)}$ in Eq.~\eqref{eq:poly} are:
\begin{align*}
d_{0:(n-1)} = x_0, \qquad d_{n:(2n-1)} = \delta_T^\T W_T^{-\T}e^{AT}H^\T
\end{align*}
where $H \in \mathbb{R}^{(3n) \times (3n)}$ with $H_{ij} = \begin{cases} (-1)^j, & i=j \\ 0, & i\neq j \end{cases}.$
\end{proposition}
Thus, the optimal cost $h_2(x_0)$ obtained in Prop.~\ref{prop:LQMT} is a better heuristic for Problem~\ref{prob:2} than $h_1$ because $h_2$ takes the control efforts into account. It is also admissible by construction because the optimal cost of Problem~\ref{prob:2} is lower bounded by the optimal cost of Problem~\ref{prob:1}, which in turn is lower bounded by $h_2(x_0)$. Below, we give examples of the results in Prop.~\ref{prop:LQMT} for several practical cases with a given $T$.

\paragraph{Velocity Control}
Let $n=1$ so that $\mathcal{X} \subset \mathbb{R}^3$ is position space and $\mathcal{U}$ is velocity space. Then, the optimal solution to Eq.~\eqref{eq:LQMT} according to Prop.~\ref{prop:LQMT} is:
\begin{align*}
d_1 &= \frac{1}{T}\prl{p_f-p_0}\\
x^*(t) &= d_1t + p_0,\;u^*(t) = d_1\\
C^* &= \frac{1}{T} \|p_f-p_0\|^2 + \rho T
\end{align*}
\paragraph{Acceleration Control}
Let $n=2$ so that $\mathcal{X} \subset \mathbb{R}^6$ is position-velocity space and $\mathcal{U}$ is acceleration space. Then, the optimal solution to Eq.~\eqref{eq:LQMT} according to Prop.~\ref{prop:LQMT} is:
\begin{align*}
\begin{pmatrix}
d_3\\d_2
\end{pmatrix} &= \begin{bmatrix}
-\frac{12}{T^3} & \frac{6}{T^2}\\
\frac{6}{T^2} & -\frac{2}{T}
\end{bmatrix}\begin{bmatrix}
p_f-p_0-v_0T\\
v_f-v_0
\end{bmatrix}\nonumber\\
x^*(t) &= \begin{bmatrix}
  \frac{d_3}{6}t^3+\frac{d_2}{2}t^2 + v_0 t + x_0\\
  \frac{d_3}{2}t^2+d_2t + v_0
\end{bmatrix},\;u^*(t) = d_3t+d_2\\
C^* &= \frac{12\|p_f - p_0\|^2}{T^3} - \frac{12(v_0+v_f)\cdot(p_f-p_0)}{T^2} + \nonumber\\
&\frac{4(\|v_0\|^2+v_0\cdot v_1+\|v_1\|^2)}{T}  + \rho T
\end{align*}

Here the optimal cost $C^*$ turns out to be a polynomial function of $T$, we are able to derive the optimal $T^*$ by minimizing $C^*(T)$ as
\begin{align*}
T^* &= \argmin_{T}C^*(T)\\
s.t. \ T& \geq \bar{T}
\end{align*}
the solution of which is the positive real root of $C^*(T)' = 0$. Furthermore, the optimal cost is $C^* = C^*(T^*)$.

\subsection{Collision Checking}
\label{sec:col}

For a calculated edge $e(t) = [p(t)^\T, v(t)^\T, a(t)^\T, ...]^\T$ in Alg.~\ref{alg:succ}, we need to check if $e(t)\subset \mathcal{X}^{free}$ for $t \in [0,\tau]$. We check collisions in the geometric space $\mathcal{P}^{free} \subset \mathbb{R}^3$ separately from enforcing the dynamic constraints $\mathcal{D}^{free} \subset \mathbb{R}^{3(n-1)}$. The edge $e(t)$ is valid only if its geometric shape $p(t) \subset \mathcal{P}^{free}$ and derivatives $(v(t), a(t), ...) \subset \mathcal{D}^{free}$, i.e.,
\begin{equation}
(v,a, ...)\subset \mathcal{D}^{free}   \Leftrightarrow\left. \begin{array}{ll}
         \|v\|_{\infty} \leq v_{max}, & \forall t\in[0, \tau]\\
         \|a\|_{\infty} \leq a_{max}, & \forall t\in[0, \tau]\\
         \vdots&\end{array}
         \right.
\end{equation}
Since the derivatives $v, a, ...$ are polynomials, we calculate their extrema within the time period $[0, \tau]$ to compare with maximum bounds on velocity, acceleration, etc. For $n \leq 3$, the order of these polynomials is less than $5$, which means we can easily solve for the extrema in closed form.

The more challenging part is checking collisions in $\mathcal{P}^{free}$. In this work, we model $\mathcal{P}$ as an \textit{Occupancy Grid Map}. Other representations such as a \textit{Polyhedral Map} are also possible but these are usually hard to obtain from real-world sensor data~\cite{liu2016high,deits2015computing} and out of the scope of the discussion in this paper. Let $P := \{p(t_i)\;|\; t_i \in [0,\tau], i = 0,\ldots, I\}$ be a set of positions that the system traverses along the trajectory $p(t)$. To ensure a collision-free trajectory, we just need to show that $p(t_i) \in \mathcal{P}^{free}$ for all $i \in \crl{0,\ldots, I}$. Given a polynomial $p(t),\; t \in [0,\tau]$, the positions $p(t_i)$ are sampled by defining:
\begin{equation}
 t_i := \frac{i}{I}\tau \quad\text{ such that }\quad \frac{\tau}{I}v_{max} \geq {R}.
\end{equation}
Here $R$ is the occupancy grid resolution. The condition ensures that the maximum distance between two consecutive samples will not exceed the map resolution. It is an approximation, since it can miss cells that are traversed by $p(t)$ with a portion of the curve within the cell shorter than $R$, but it prevents the trajectory from hitting obstacles.



\section{Trajectory Refinement}
\label{sec: ref}
A trapezoid velocity profile is widely used to describe the robot following a path, in which the robot is assumed to move as a particle that exactly tracks the path with defined velocity function. This model gives the so-called \textit{time allocation} for a large group of trajectory optimization approaches described in~\cite{mellingerICRA2011},~\cite{richter2016polynomial},~\cite{liu2016high},~\cite{liu2017plan} and~\cite{chen2016online}. However, this approximation is naive and the resulting trajectory significantly deforms from the given path since the modeled particle is not obeying the expected dynamics.

In above section, we proposed the complete solution for planning a trajectory that is valid in control space. The resulting trajectory gives not only the collision-free path, but also the time for reaching those waypoints. Thus, we are able to use it as a prior to generate a smoother trajectory in higher dimension for controlling the actual robot. The refined trajectory $x^*(t)$ is derived from solving an unconstrained QP with given initial and end states $s_0, s_g$ and the intermediate waypoints $p_k, k \in \crl{0,\ldots,N-1}$.

\begin{equation}
  \label{eq: problem4}
  \begin{gathered}
    \min_{D} \; \sum_{k = 0}^{N-1}\int_{0}^{\tau_k} \left\|p_{Dk}^{(n)}(t)\right\|^2 dt\\
    \begin{aligned}
      \text{s.t.}\;&x_0(0) = s_0,\quad x_{N-1}(\tau_{N-1}) = s_g\\
      &x_{k+1}(0) = x_k(\tau_k),\; k \in \crl{0,\ldots, N-2}\\
      &p_{Dk}(\tau_k) = p_k, \qquad k \in \crl{0,\ldots, N-1}
    \end{aligned}
  \end{gathered}
\end{equation}

The time for each trajectory segment $\tau_k$ is also given from the prior trajectory. The solution for Eq.~\eqref{eq: problem4} is proposed in~\cite{mellingerICRA2011}. We ignore the mathematical details in this section and only show the trajectory refinement results in Fig.~\ref{fig: compare}.

 \begin{figure}[htp]
  \centering
  \subfigure[$T = 8.5$.]{\includegraphics[width=0.45\columnwidth, trim=0 0 0 0, clip=true]{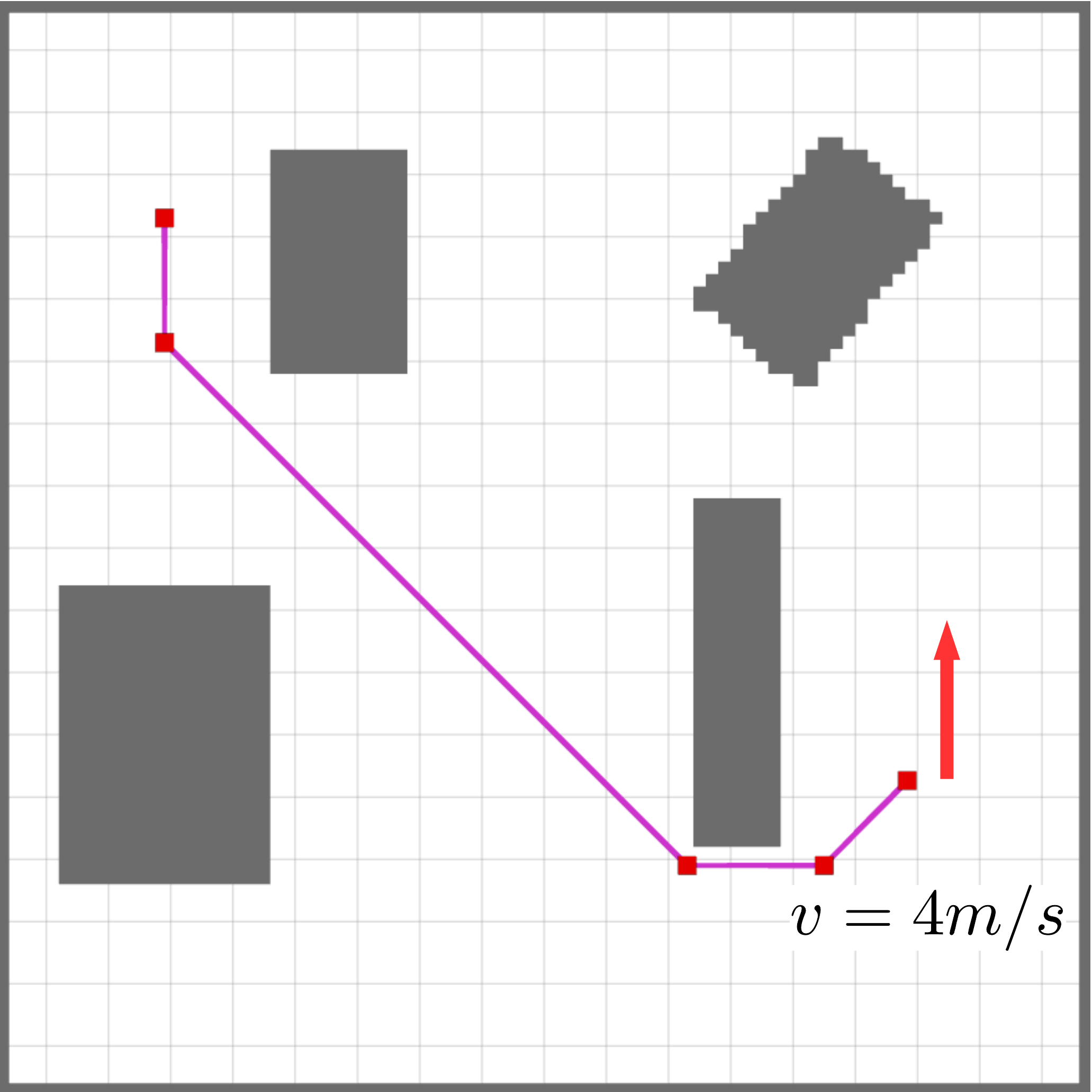}\label{fig: traj_pos1}}
  \subfigure[$T = 8.5, J = 296.6$.]{\includegraphics[width=0.45\columnwidth, trim=0 0 0 0, clip=true]{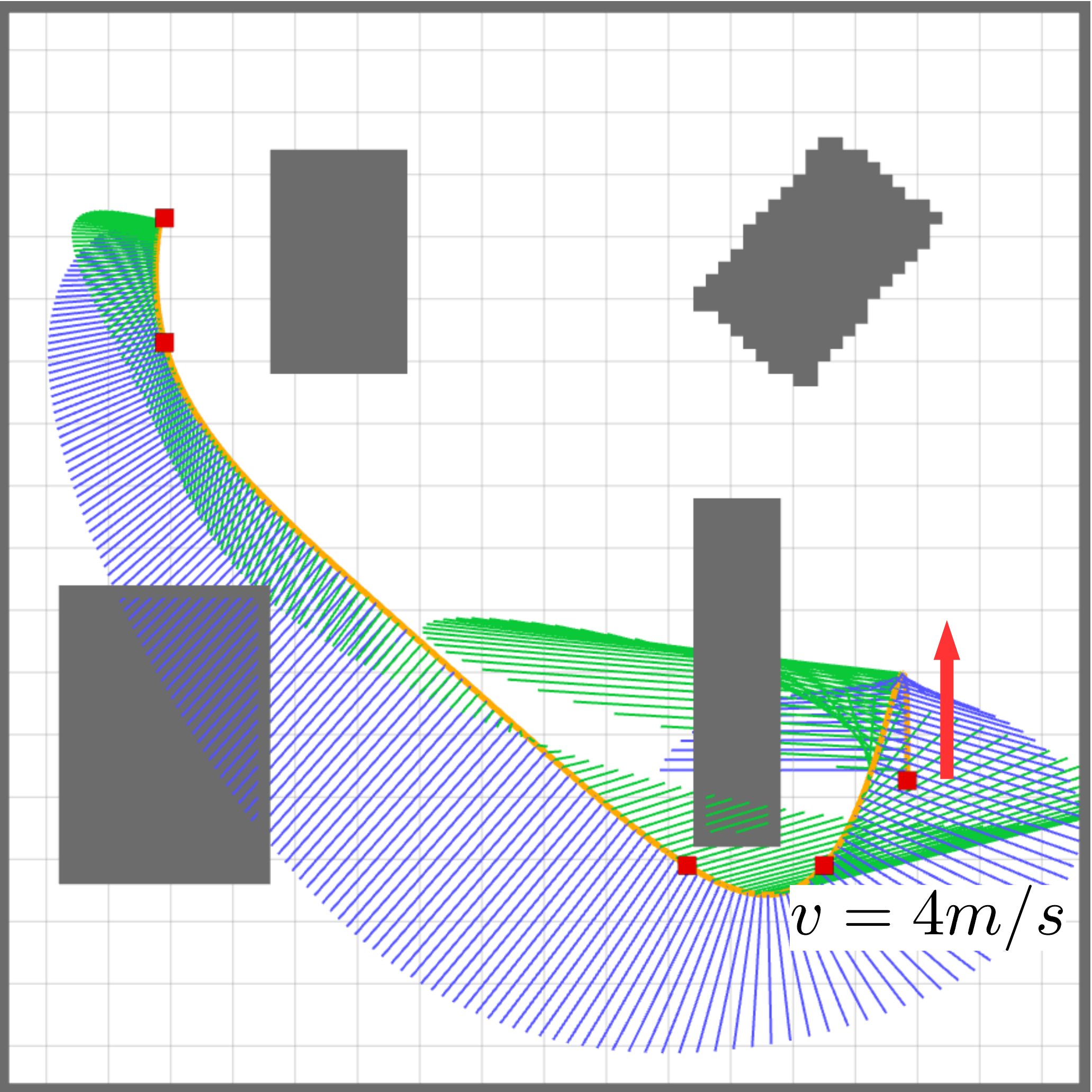}\label{fig: traj_pos2}}
  \subfigure[$T = 10, J = 14.0$.]{\includegraphics[width=0.45\columnwidth, trim=0 0 0 0, clip=true]{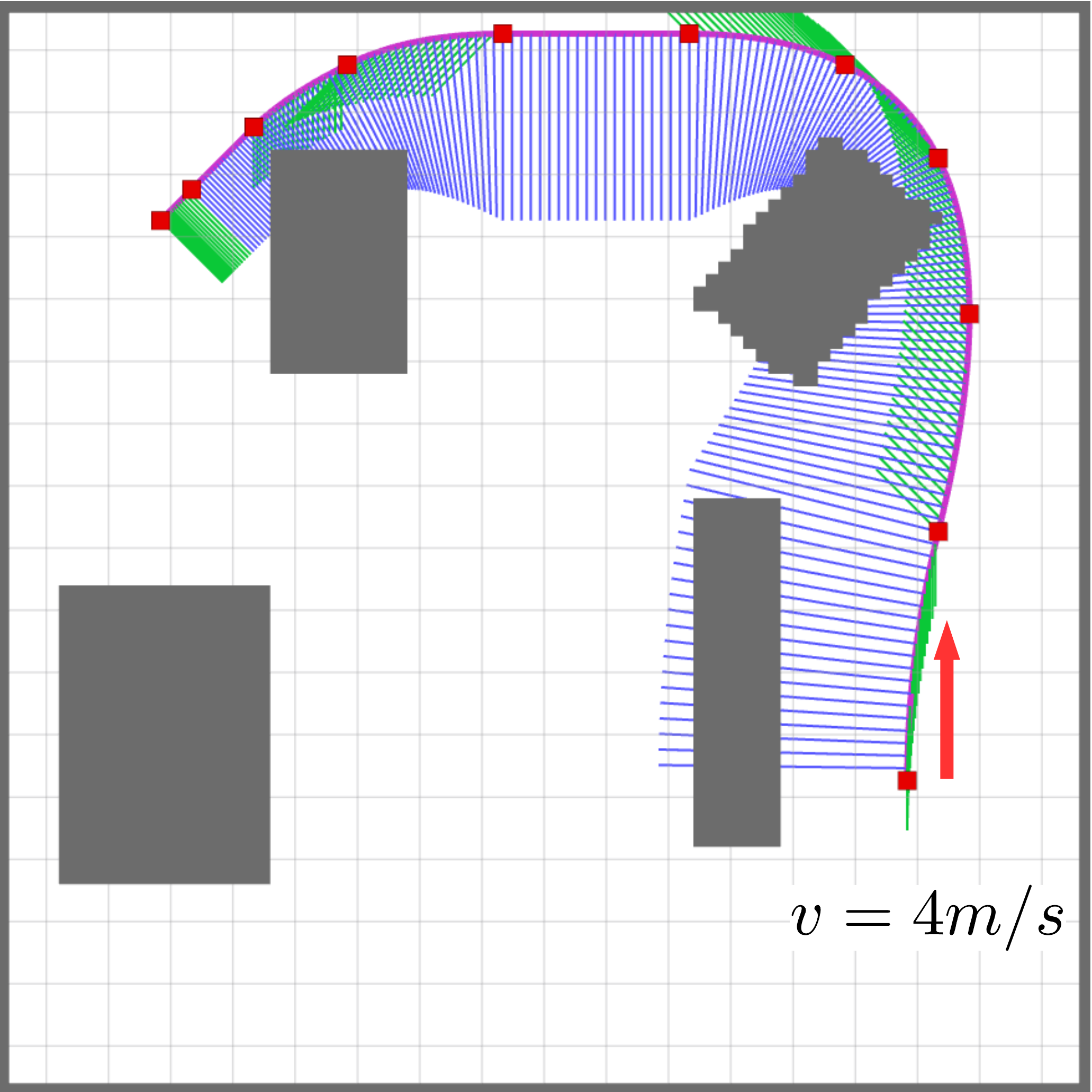}\label{fig: traj_vel1}}
  \subfigure[$T = 10, J = 21.3$.]{\includegraphics[width=0.45\columnwidth, trim=0 0 0 0, clip=true]{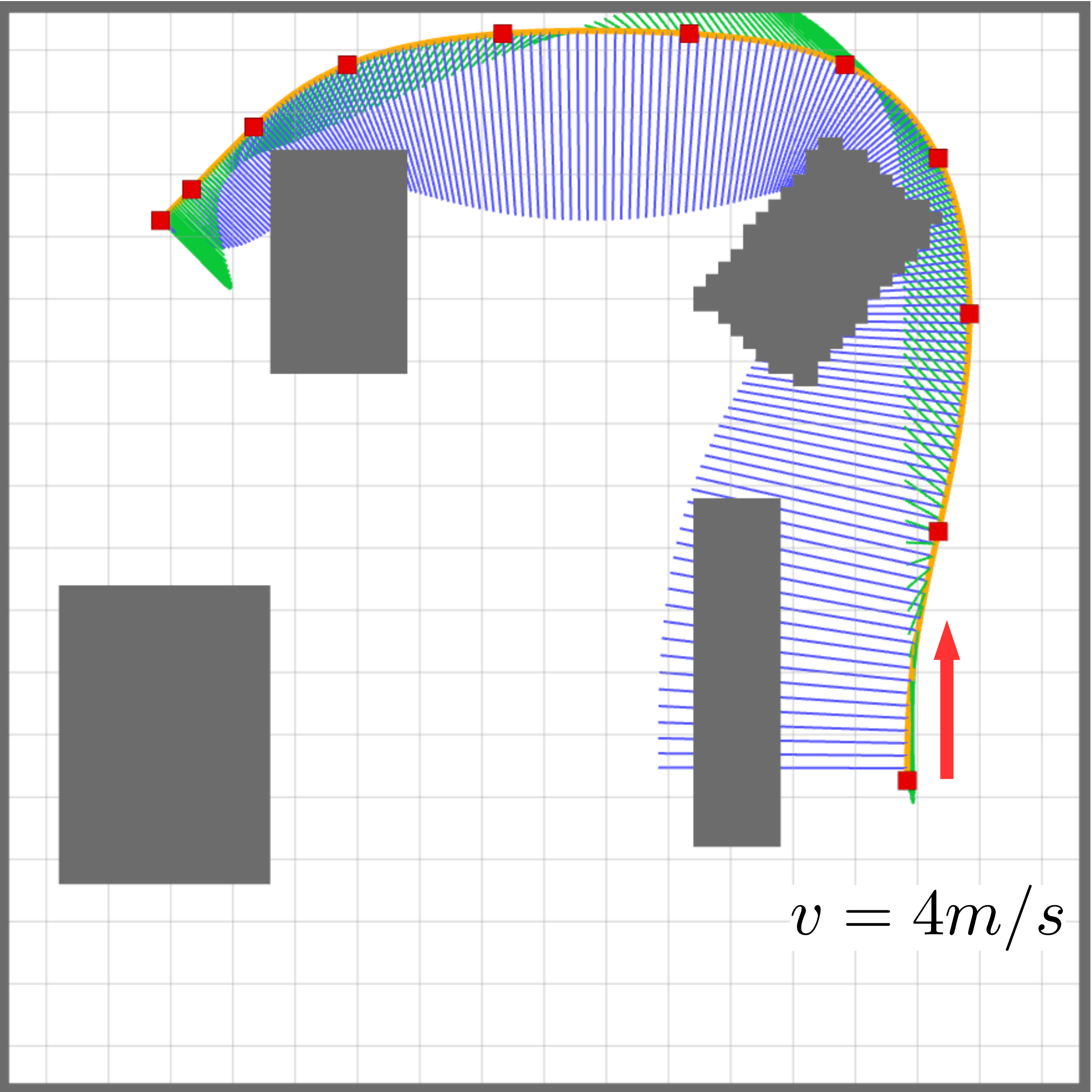}\label{fig: traj_vel2}}
  \subfigure[$T = 12, J = 11.3$.]{\includegraphics[width=0.45\columnwidth, trim=0 0 0 0, clip=true]{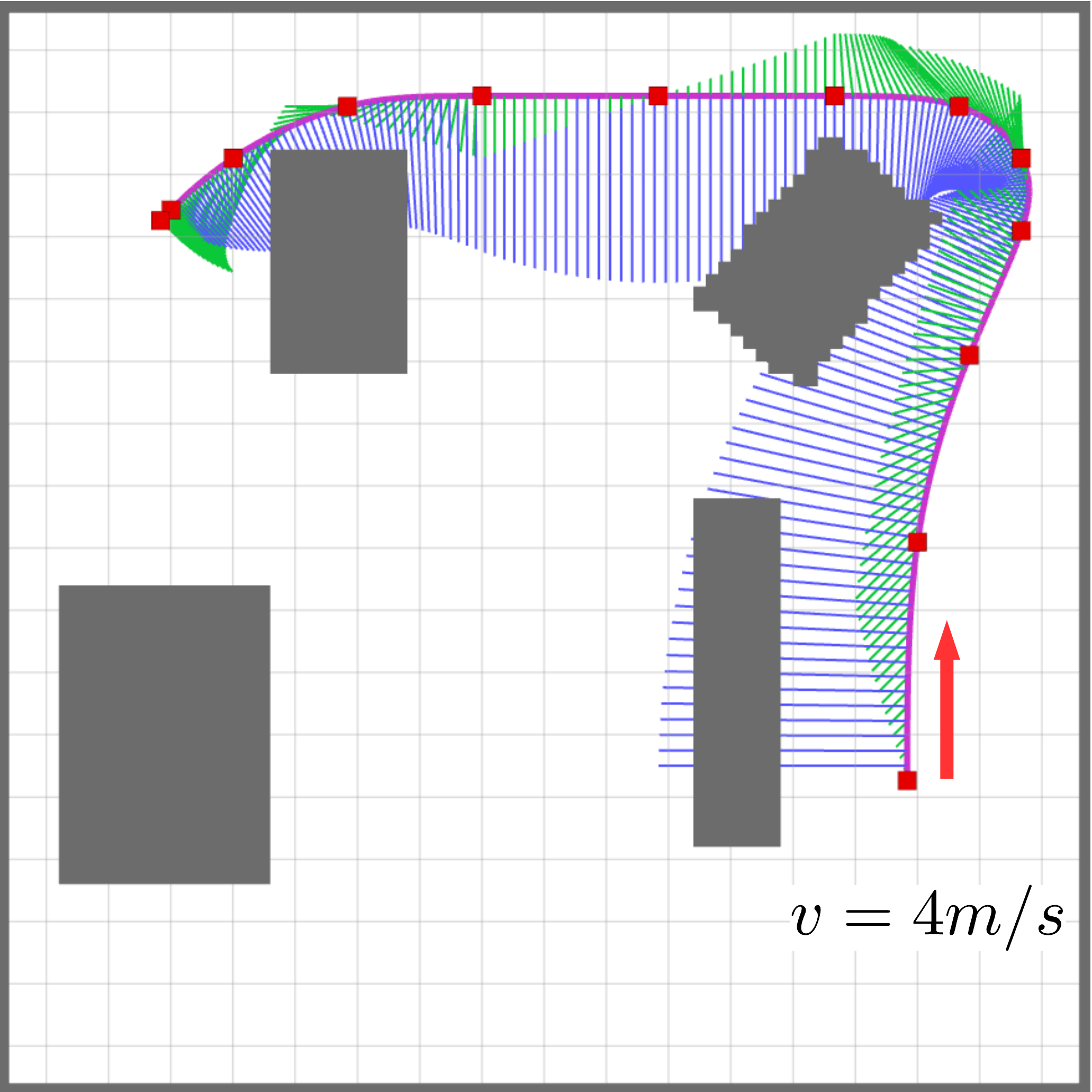}\label{fig: traj_acc1}}
  \subfigure[$T = 12, J = 13.6$.]{\includegraphics[width=0.45\columnwidth, trim=0 0 0 0, clip=true]{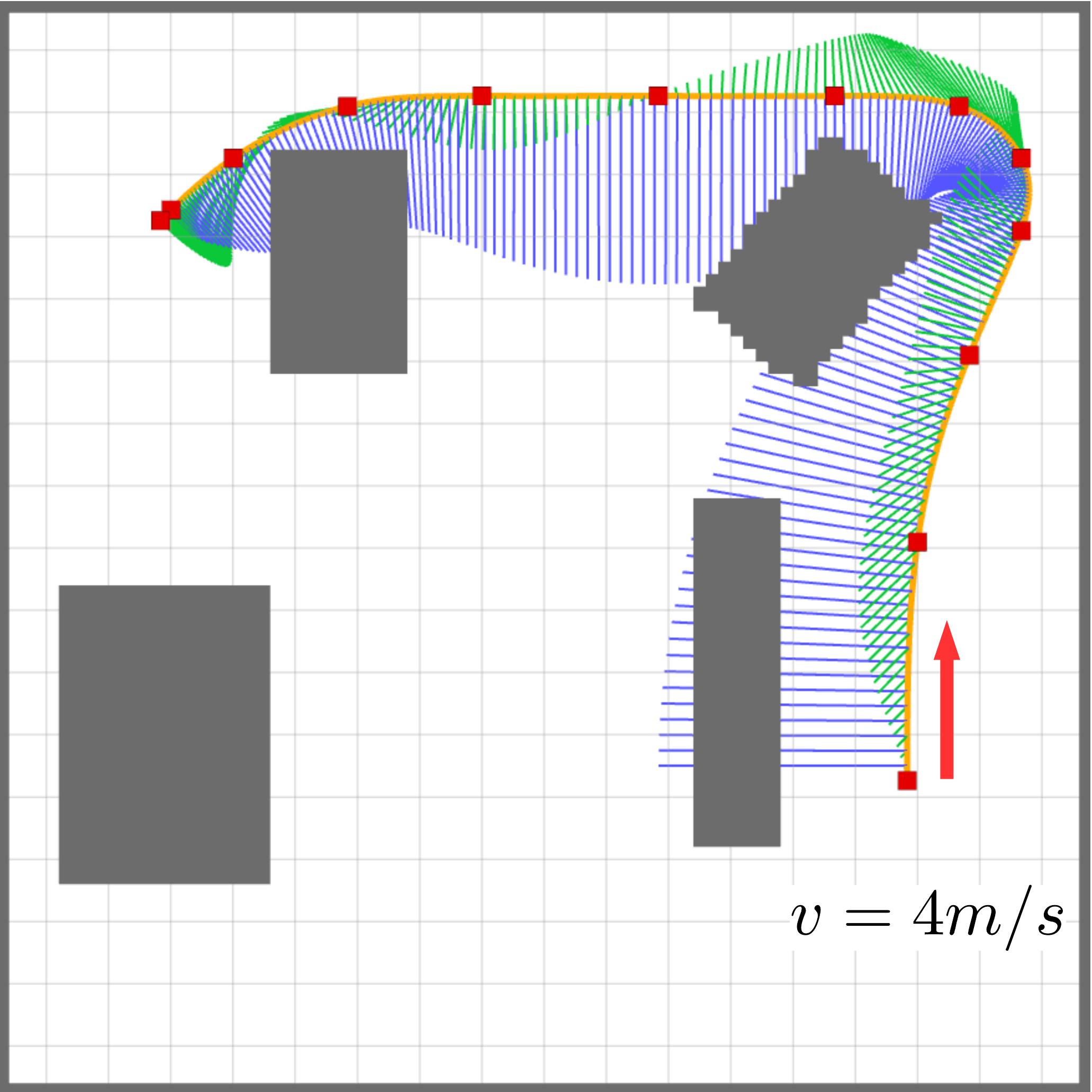}\label{fig: traj_acc2}}
   \caption{Trajectories planned from start $s$ to goal $g$ with initial velocity ($4m/s$). The blue/green lines show the speed/acceleration along trajectories respectively and the red points are the intermediate waypoints. (a) shows the shortest path. The time is allocated using the trapezoid velocity profile for generating min-jerk trajectory in (b). The resulting trajectory has a large cost for efforts $J$. (c) shows the trajectory planned using acceleration-controlled system. In this case, the acceleration is not continuous. In (d), we refine using a min-jerk trajectory which has continuous and smooth acceleration. (e) shows the trajectory planned using jerk-controlled system. The acceleration is continuous but not smooth. In (f), the refined min-jerk trajectory has continuous and smooth acceleration. \label{fig: compare}}
\end{figure}
It needs to be notified that even though the refinement step produces a smoother trajectory, the refined trajectory might be unsafe and infeasible. 

\section{Experimental Results}
\label{sec:exp}
\subsection{Heuristic Function}
We proposed two different heuristics in Sec.~\ref{sec:heur}: denote the first one that estimates the minimum time using the max speed constraint as $h_1$; denote the other one estimates the minimum cost function using the dynamic constraints as $h_2$. The heuristic $h_1$ is easier to compute, but it fails to take in to account of the system's dynamics; the heuristic $h_2$ requires to solve for the real roots of a polynomial, but it reveals the lower bound of the cost regarding system's dynamics and thus it is a tighter underestimation of the actual cost. Here we compare the performance of the algorithm with respect to the two heuristics $h_1, h_2$. As a reference, by setting the heuristic function to zero changes the algorithm into Dijkstra search. Fig.~\ref{fig:hs} visualizes the expanded nodes while searching towards the goal from a state with initial velocity $3m/s$ in positive vertical direction.

\begin{figure}[htp]
  \centering
  \subfigure[Dijkstra. $T_p = 0.16s, N_p = 2707$]{\includegraphics[width=0.32\columnwidth, trim=0 0 0 0, clip=true]{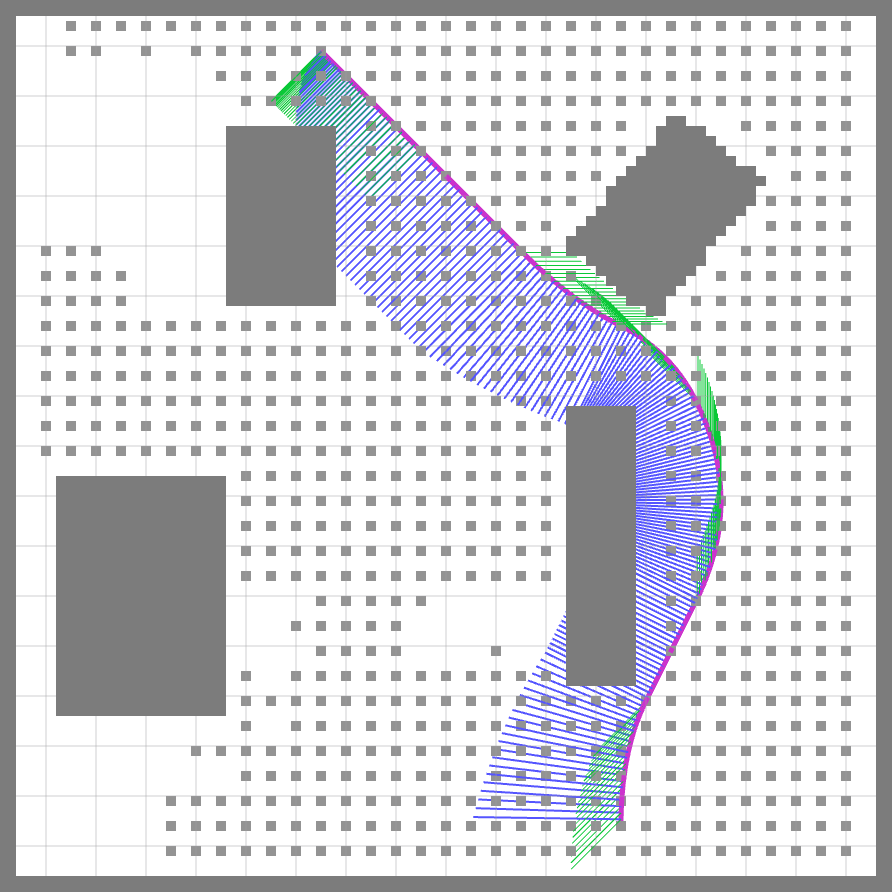}\label{fig: h0}}
  \subfigure[$A^*$ with $h_1$. $T_p = 0.064s, N_p = 1282$]{\includegraphics[width=0.32\columnwidth, trim=0 0 0 0, clip=true]{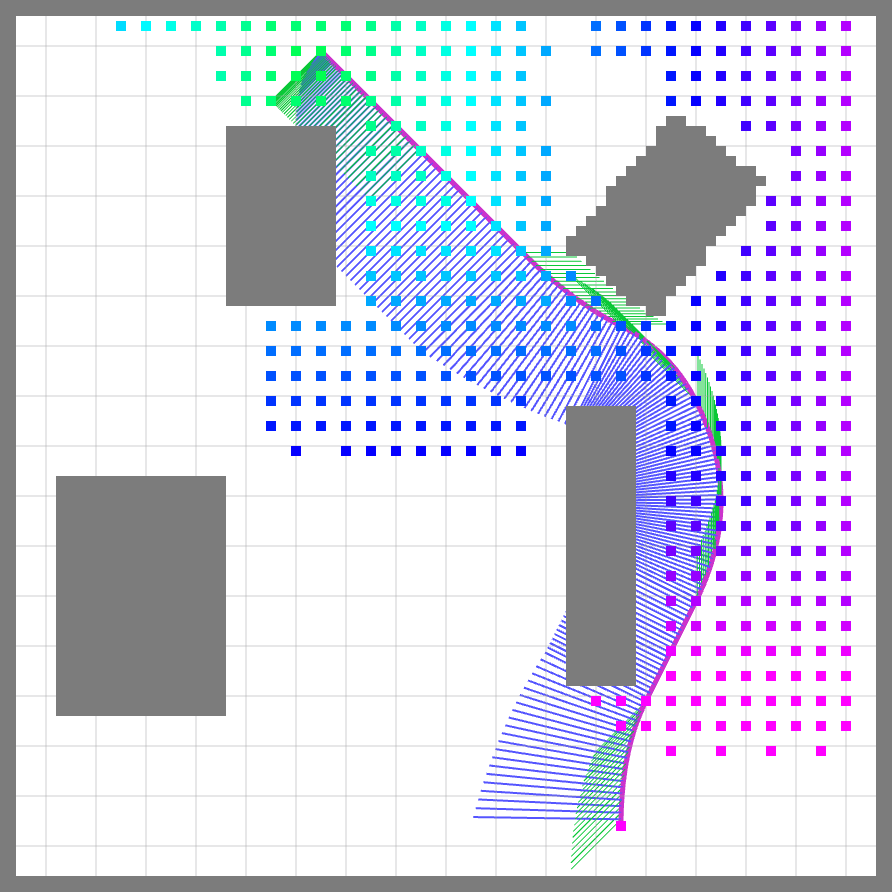}\label{fig: h1}}
  \subfigure[$A^*$ with $h_2$. $T_p = 0.016s, N_p = 376$]{\includegraphics[width=0.32\columnwidth, trim=0 0 0 0, clip=true]{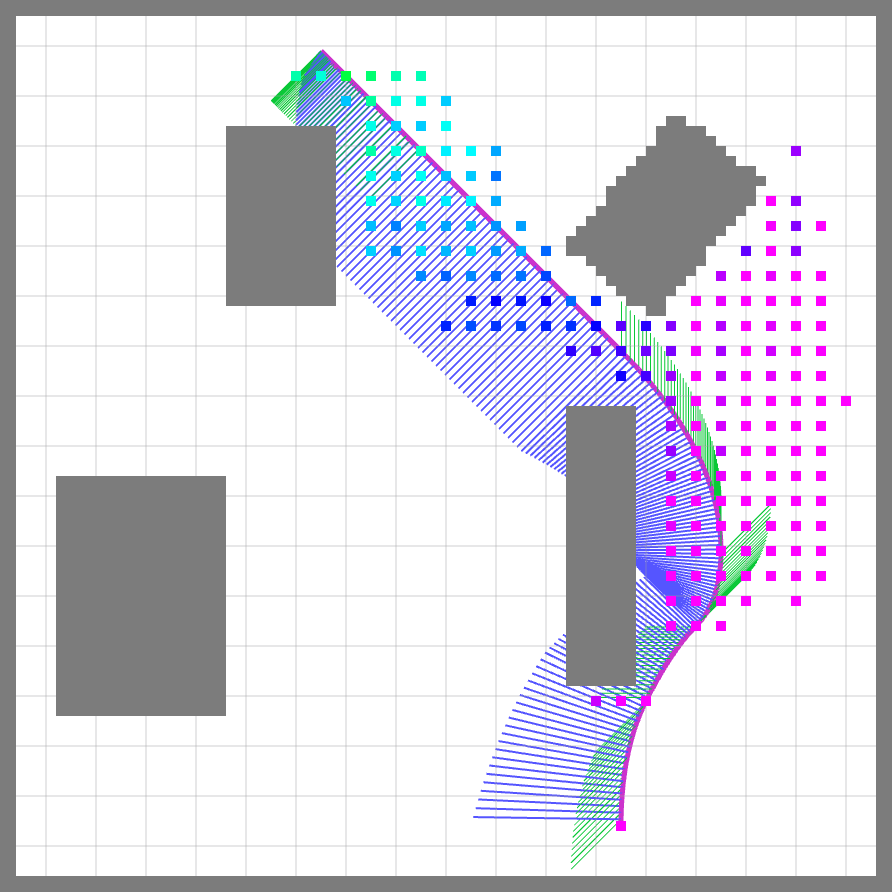}\label{fig: h2}}
   \caption{Generated trajectories using different heuristics. The expanded nodes (small dots) are colored by the corresponding cost value of the heuristic function. Grey nodes have zero heuristic value, high cost nodes are colored red while low cost nodes are colored green. $T_p$ and $N_p$ shows the time for planning and number of expanded nodes respectively. \label{fig:hs}}
\end{figure}


We can see that the \textit{Minimum Cost Heuristic} $h_2$ makes the searching faster as it expands less nodes without loss of optimality. However, when it comes to the system with higher dimension, calculating $h_2$ becomes harder as one can not analytically find the roots for a polynomial with order greater than $4$. As claimed in Sec.~\ref{sec:heur}, when the maximum velocity is low, $h_1$ is efficient enough for any dynamic system.

\subsection{Run Time Analysis}
To evaluate the computational efficiency of the algorithm, we record the run time of generating hundreds of trajectories (Fig.~\ref{fig: maps}) using either acceleration-controlled or jerk-controlled system in both 2-D and 3-D environments. Table~\ref{table:time2} shows the time it takes for each system. We can see that planning in 3-D takes more time than in 2-D; also, planning in jerk space is much slower ($10$ times) than in acceleration space.
\begin{figure}[htp]
  \centering
  \subfigure[2-D Planning.]{\includegraphics[width=0.45\columnwidth, trim=0 0 0 0, clip=true]{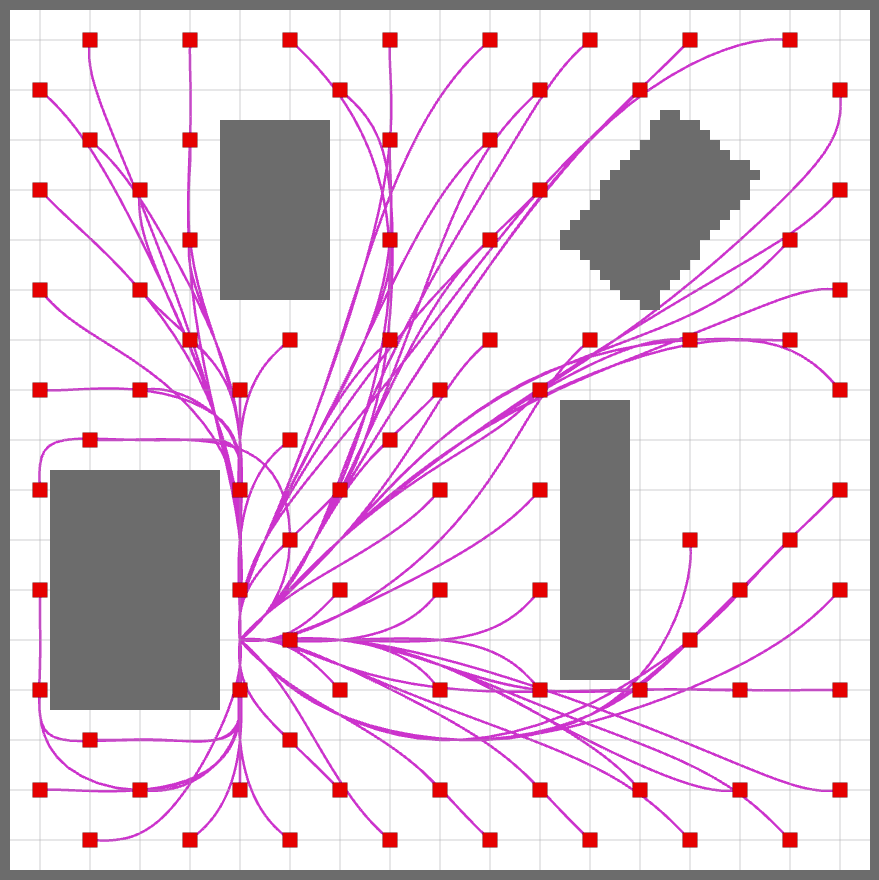}\label{fig: m1}}
  \subfigure[3-D Planning.]{\includegraphics[width=0.45\columnwidth, trim=0 0 0 0, clip=true]{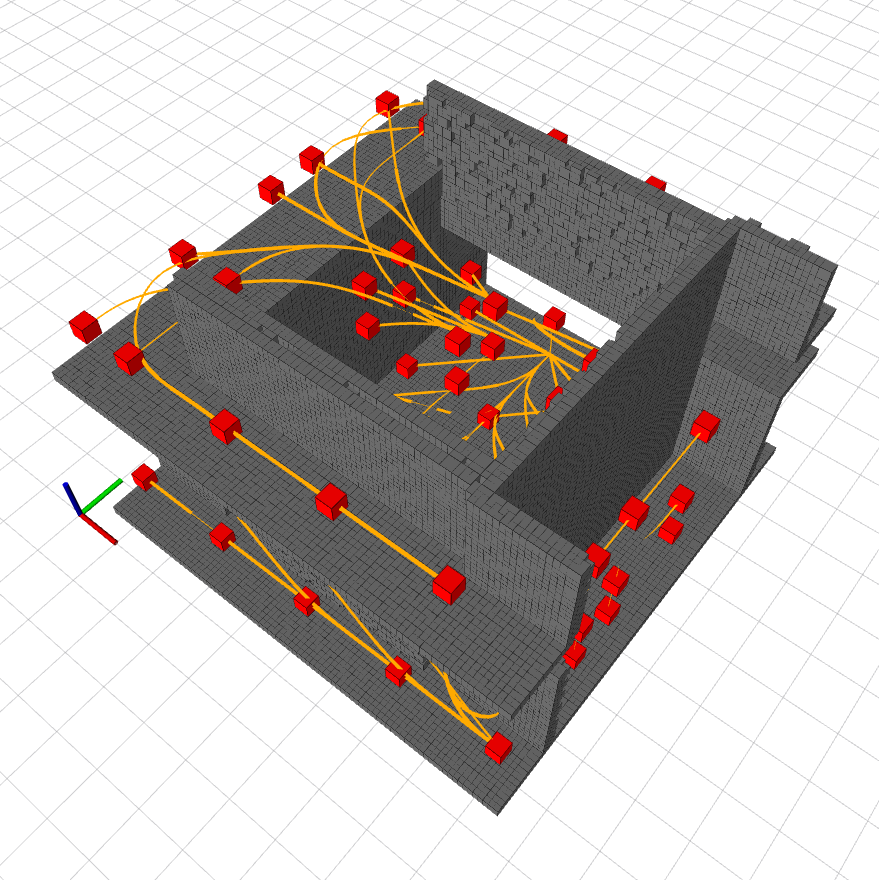}\label{fig: m2}}
   \caption{Trajectories generated to sampled goals (small red balls). For 2-D case, we use 9 primitives while for 3-D case, the number is 27. \label{fig: maps}}
\end{figure}

\begin{table}[htb]
\centering
  \caption{Trajectory Generation Run Time} \label{table:time2}
	\begin{tabular}{ c | c | c | c }
	\hline
	Map & Time(s) & Accel-controlled & Jerk-controlled\\\hline
	\multirow{3}{*}{2-D} & Avg & 0.016 & 0.147\\
	 &  Std & 0.015 & 0.282 \\
	 &  Max & 0.086 & 2.13\\\hline
	 \multirow{3}{*}{3-D} & Avg & 0.094 & 2.98\\
	 &  Std & 0.155 & 3.78 \\
	 & Max & 0.515 & 9.50\\\hline
	\end{tabular}
\end{table}


\subsection{Re-planning and Comparisons}
\textit{Receding Horizon Control} (RHC) has been widely used for navigating an aerial vehicle in unknown environments~\cite{howACC02}, the frequently re-planning process allows the robot to keep moving with limited sensing range until it reaches the goal region. In this section, we show results of our navigation system that builds on the RHC framework with the proposed trajectory generation method. As a comparison, we also set up the system that utilizes the prior planned path as the guide for trajectory generation. To demonstrate the fully autonomous collision avoidance on a quadrotor, we use the AscTec Pelican platform with a Hokuyo laser range-finder. We run state estimation and obstacle detection (mapping) as described in~\cite{shenICRA2012} on an onboard Intel NUC-i7 computer. Fig.~\ref{fig: exps} shows the performance of using these two approaches to avoid an obstacle by re-planning at the circle position where the desired speed is non-zero. The traditional path-based approach in Fig.~\ref{fig: exp1} leads to a sharp turn while our approach generates a smoother trajectory shown in Fig.~\ref{fig: exp2}.

\begin{figure}[htp]
  \centering
   \subfigure[Experiment environment.]{\includegraphics[width=0.9\columnwidth, trim=0 0 0 0, clip=true]{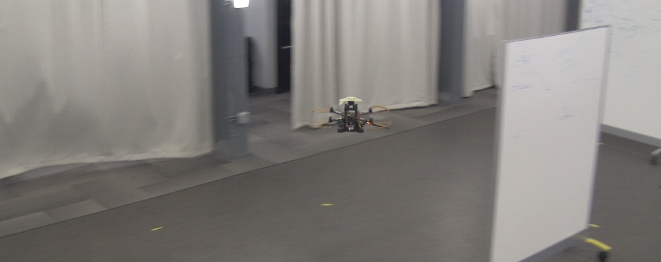}\label{fig: exp0}}
  \subfigure[Re-plan with path-based approach.]{\includegraphics[width=0.9\columnwidth, trim=0 0 0 0, clip=true]{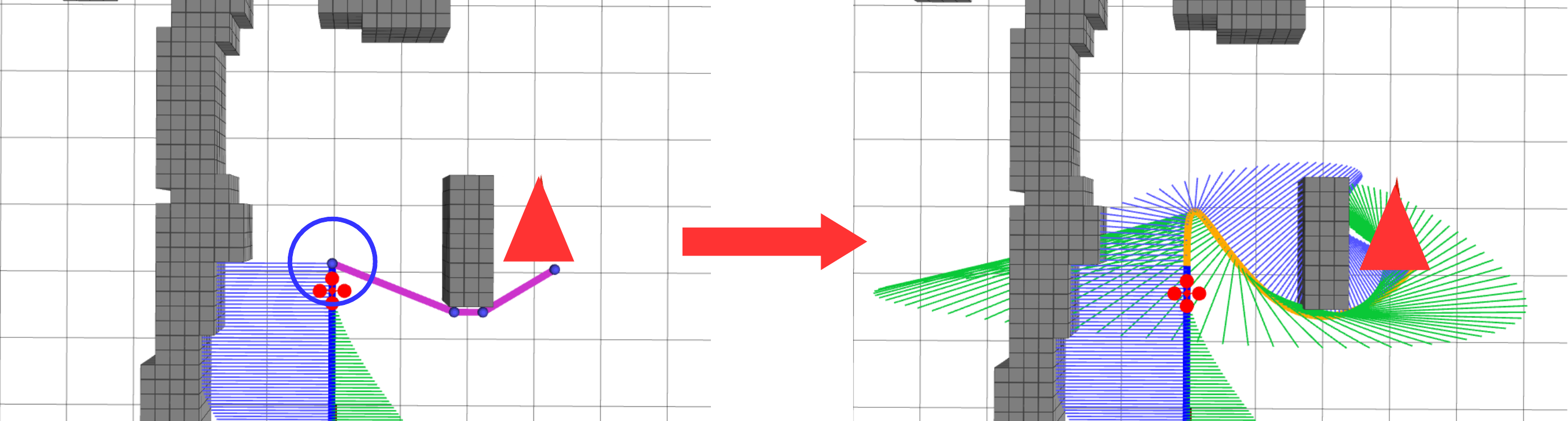}\label{fig: exp1}}
  \subfigure[Re-plan with our method.]{\includegraphics[width=0.9\columnwidth, trim=0 0 0 0, clip=true]{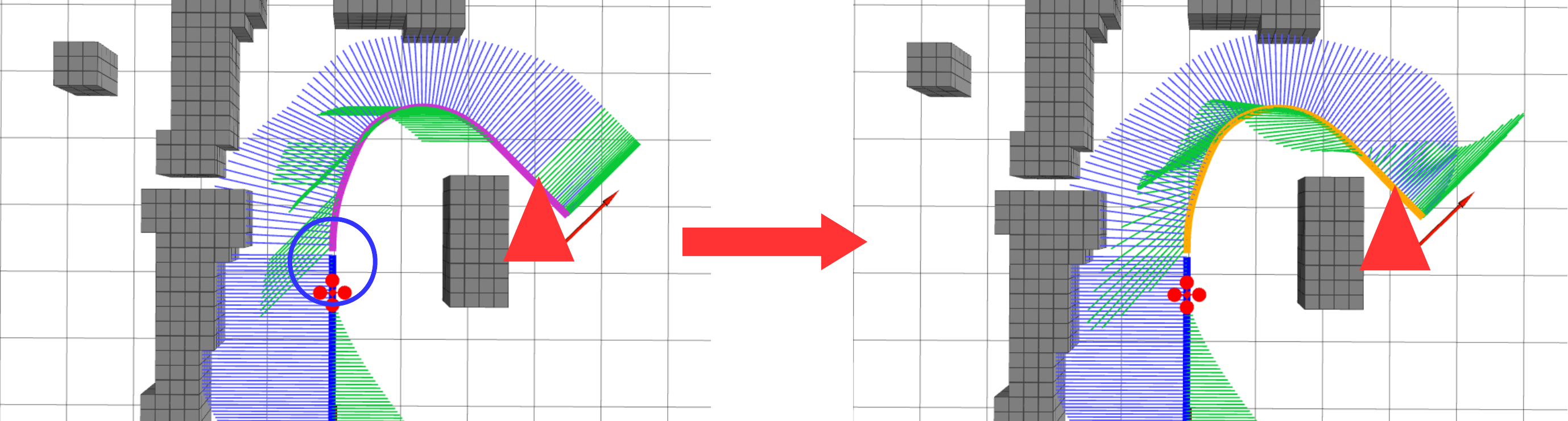}\label{fig: exp2}}
   \caption{Pelican experiments using different trajectory generation pipelines. The robot is initially following a trajectory (blue curve) and needs to re-plan at the end of this prior trajectory (circled) to go to the goal (red triangle). The state from which the robot re-plans is non-static and the speed is $2m/s$ in positive vertical direction. (b) shows the result of using traditional path-based trajectory generation method, the shortest path (purple line segments in the left figure) leads to the final trajectory (yellow curve in the righ figure); (c) shows the result of using our trajectory generation method, the shortest trajectory (purple curve in the left figure) leads to the smoother final trajectory (yellow curve in the righ figure). \label{fig: exps}}
\end{figure}

Fig.~\ref{fig:ts} shows the results in simulation where we set up a longer obstacle-cluttered corridor for testing. The re-planning is triggered constantly at $3Hz$ and the maximum speed is set to be $3m/s$. Our method generates a better overall trajectory compared to the traditional method as it avoids sharp turns when avoiding obstacles.

\begin{figure}[htp]
  \centering
  \subfigure[Simulation Environment.]{\includegraphics[width=0.9\columnwidth, trim=0 0 0 0, clip=true]{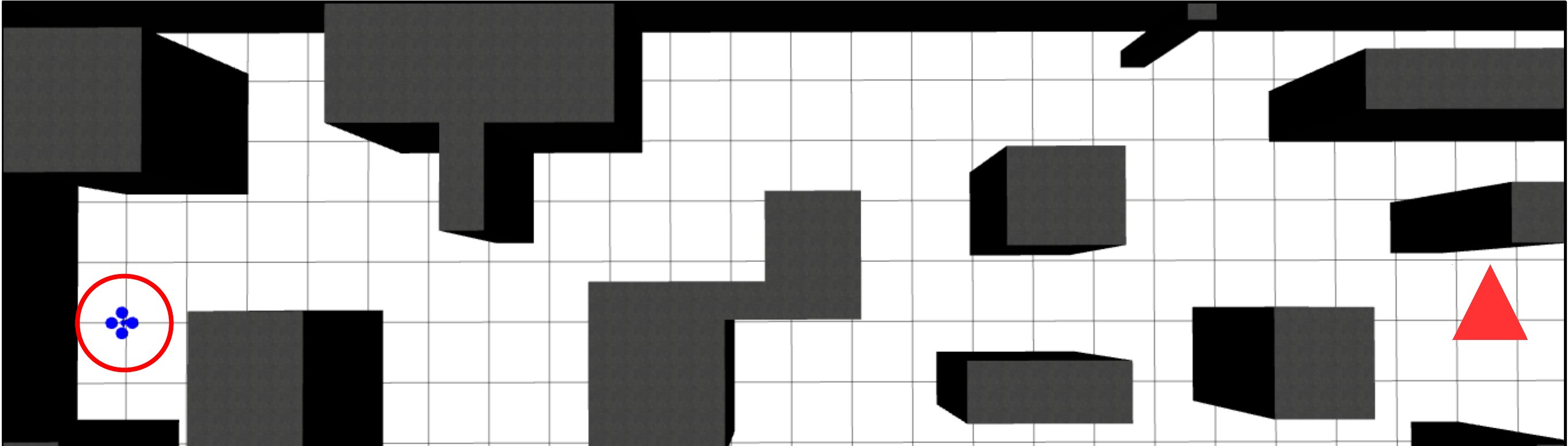}\label{fig: env}}
  \subfigure[Path-based approach.]{\includegraphics[width=0.48\columnwidth, trim=0 0 0 0, clip=true]{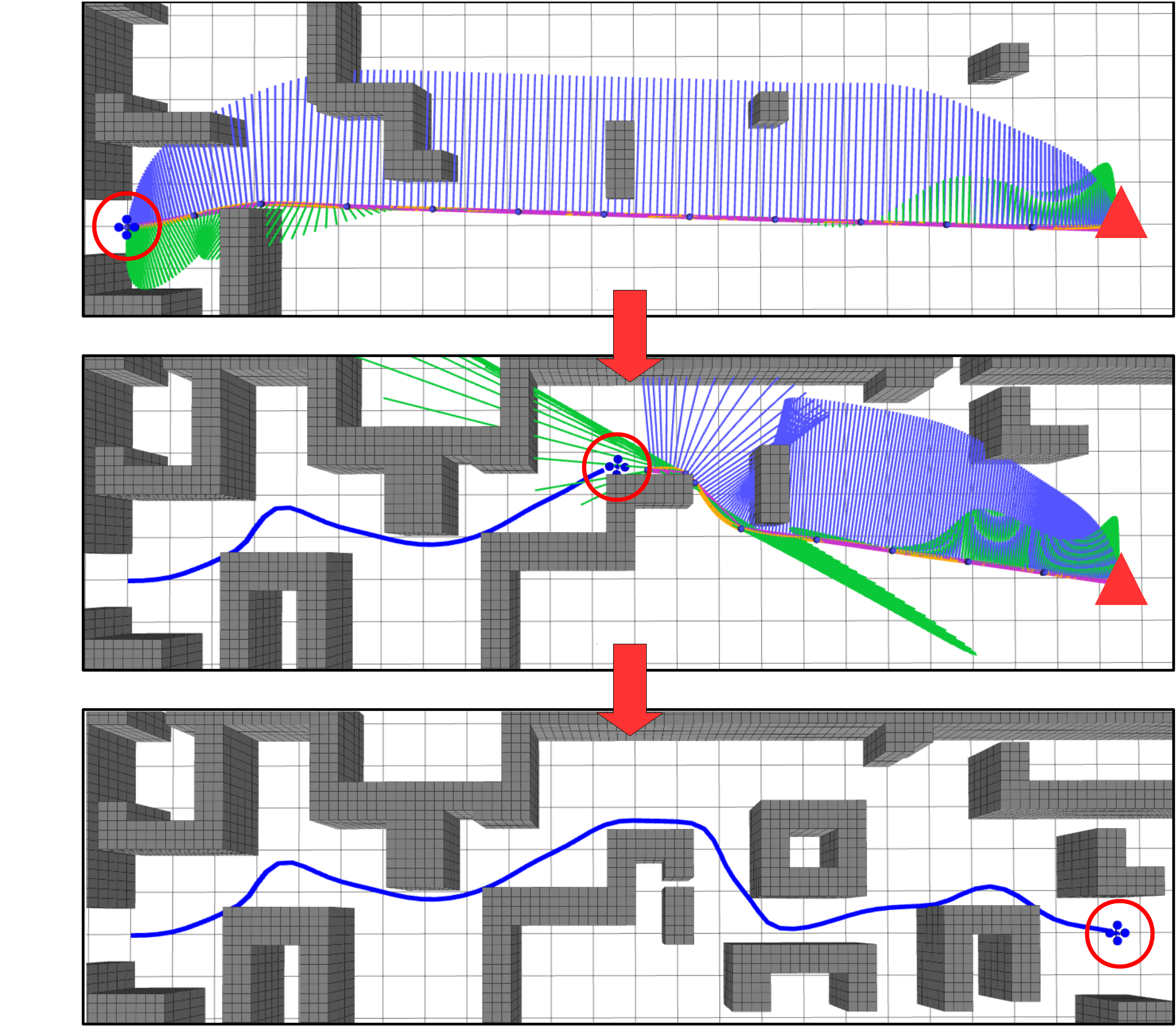}\label{fig: to1}}
  \subfigure[Our method.]{\includegraphics[width=0.48\columnwidth, trim=0 0 0 0, clip=true]{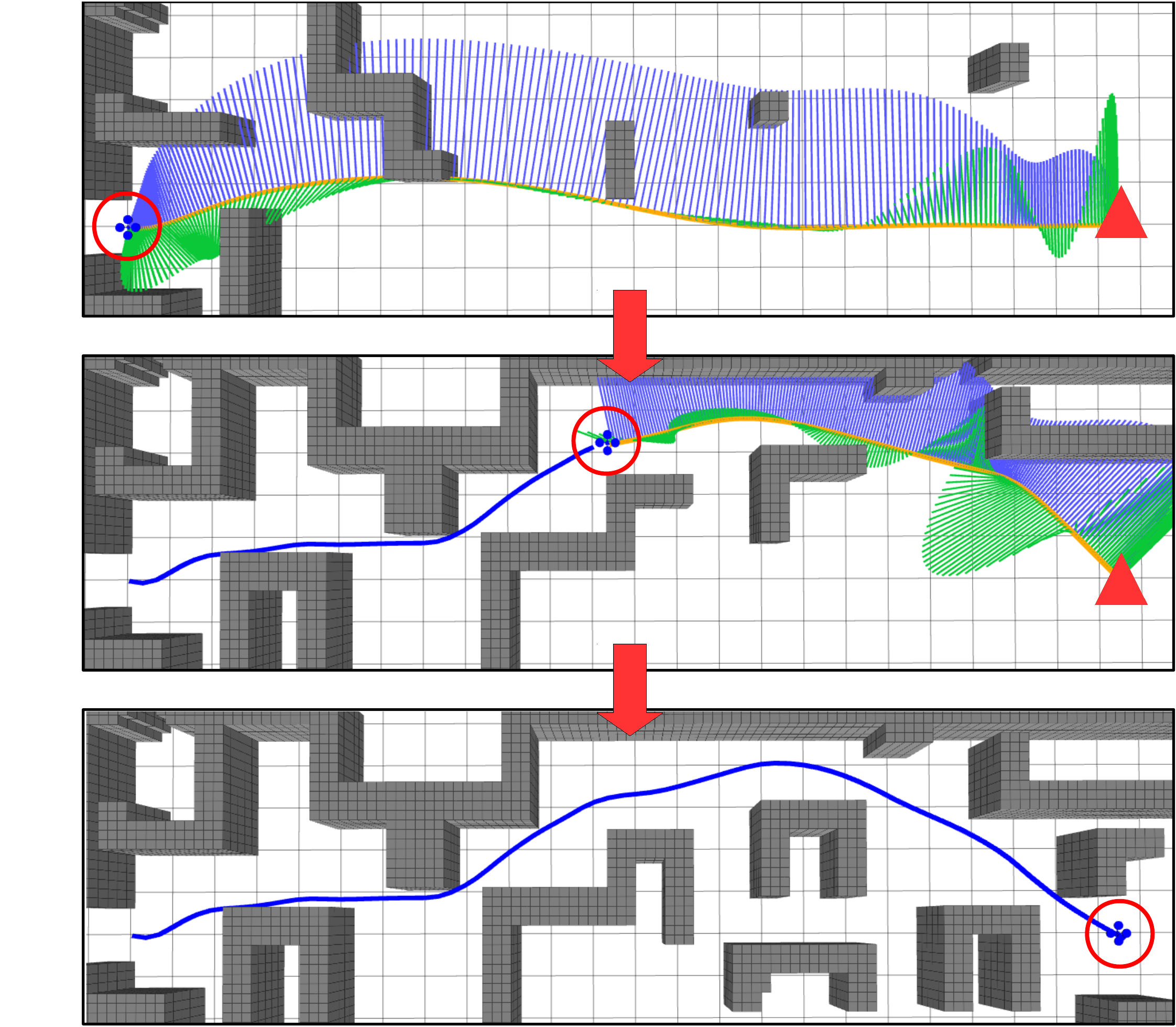}\label{fig: tn1}}
   \caption{Re-planning with RHC in simulation using different trajecotry generation pipelines. The robot starts from the left (circled) and the goal is at the right side of the map (red triangle). Blue curves show the traversed trajectory. (b) shows the re-planning processes using traditional path-based trajectory generation method. (c) shows the re-planning processes using proposed method in this paper. We can see that the overall trajectories in (c) is smoother than in (b). \label{fig:ts}}
\end{figure}


\section{Conclusion}
\label{sec: con}
Search-based planning is well-known to be inefficient for high dimensional planning due to the large number of nodes to expand. Even though lattice search techniques with motion primitives have been explored for ground vehicles, it is still a hard problem to consider the system's dynamics in planning phase. Using ideas from optimal control, we propose a solution that plan optimal trajectories in high dimensional spaces within a reasonable time. The experimental results reveal the success of using it as the foundation for a safe and fast navigation system for a quadrotor. The deterministic optimal trajectory helps in reducing errors in state estimation and control, saving system energy and making robot's motion predictable. We believe the basic approach proposed in this paper is valuable for planning optimal trajectories for any system that is differential flat, moreover, this generic framework can be integrated with other path planning technique like sampling-based methods to generate trajectories.

\appendices
\section{}
\label{app:a}
\begin{proof}[Proof of Prop.~\ref{prop:discretization}]
  Given an initial state $x_0$ and a sequence of $k$ inputs, $u_1,\ldots,u_k$, are applied each for time $\tau$. The final state after applying the $k$ inputs is given by,
  \begin{equation*}
    \begin{gathered}
      x(k\tau) = F^{k}(\tau) x_0 + \sum_{i = 0}^{k-1} F^{i}(\tau) G(\tau) u_{k-i}\\
      \begin{aligned}
        F^{k}(\tau) &= \brl{\begin{smallmatrix}
          \mathbf{I}_3 & k \tau \mathbf{I}_3 & \cdots & \frac{(k \tau)^{n-1}}{(n-1)!} \mathbf{I}_3\\
          \mathbf{0} & \mathbf{I}_3 & \cdots & \frac{(k \tau)^{n-2}}{(n-2)!} \mathbf{I}_3\\
          \vdots& \ddots&\ddots&\vdots\\
          \mathbf{0}& \cdots & \mathbf{I}_3 & k \tau \mathbf{I}_3\\
          \mathbf{0}& \cdots & \mathbf{0} & \mathbf{I}_3\\
        \end{smallmatrix}} \\
        F^{i}(\tau)G(\tau) &= \brl{\begin{smallmatrix}
          \brl{(i+1)^n - i^n} \frac{\tau^n}{n!}\mathbf{I}_3\\
          \brl{(i+1)^{n-1} - i^{n-1}} \frac{\tau^{n-1}}{(n-1)!}\mathbf{I}_3\\
          \vdots \\
          \brl{(i+1)^2 - i^2} \frac{\tau^2}{2!} \mathbf{I}_3\\
          \tau \mathbf{I}_3
        \end{smallmatrix}}
      \end{aligned}
    \end{gathered}
  \end{equation*}
  Our discretized inputs are of the form $u_i = d_u \kappa_i$ where $\kappa \in \mathbb{Z}^3$ leading to $x(k\tau)$ being of the form
  \begin{equation*}
    x(k\tau) = F^{k}(\tau) x_0 + \brl{\begin{smallmatrix}
      \prl{\sum_{i = 0}^{k-1} \brl{(i+1)^n - i^n} \kappa_{k-i}} d_u \frac{\tau^n}{n!}\\
      \prl{\sum_{i = 0}^{k-1} \brl{(i+1)^{n-1} - i^{n-1}} \kappa_{k-i}} d_u \frac{\tau^{n-1}}{(n-1)!}\\
      \vdots \\
      \prl{\sum_{i = 0}^{k-1} \kappa_{k-i}} d_u \tau
    \end{smallmatrix}}
  \end{equation*}
  Thus we can see that each term in the expression for $x(k\tau)$ is a variable integer times a constant which means that our state space is discretized due to discretization of the inputs.
\end{proof}

\section{}
\label{app:b}
\begin{proof}[Proof of Prop.~\ref{prop:2}]
Since the trajectory connecting $s_i$ and $s_j$ is collision-free by construction of the graph $\mathcal{G}$ (see Alg.~\ref{alg:succ}), the optimal control from $s_i$ to $s_j$ according to the cost function in \eqref{eq:problem1} has the form prescribed by Prop.~\ref{prop:LQMT}. In detail
\[
\delta_\tau = s_j - F(\tau)s_i = G(\tau) u_{ij}
\]
and the optimal control is:
\begin{align*}
u^*(t) &= B^\T e^{A^\T(\tau-t)} W_\tau \delta_\tau\\
&= B^\T e^{A^\T(\tau-t)} \left(\int_0^\tau e^{As} BB^\T e^{A^\T s} ds\right)^{-1}\int_0^\tau e^{As}ds Bu_{ij}
\end{align*}
Since only the bottom $3 \times 3$ block of $B$ is non-zero and since the matrix $e^{A^\T(\tau-t)} \left(\int_0^\tau e^{As} BB^\T e^{A^\T s} ds\right)^{-1}\int_0^\tau e^{As}ds$ has its bottom-right $3\times 3$ block equal to $I_{3 \times 3}$, we get:
\[
B^\T e^{A^\T(\tau-t)} \left(\int_0^\tau e^{As} BB^\T e^{A^\T s} ds\right)^{-1}\int_0^\tau e^{As}ds B = I_{3 \times 3}
\]
which implies that $u^*(t) \equiv u_{ij}$.
\end{proof}


\bibliographystyle{bib/IEEEtran}
\bibliography{bib/references}

\end{document}